\documentclass{article}
\usepackage{PRIMEarxiv}

\usepackage[utf8]{inputenc} 
\usepackage[T1]{fontenc}    
\usepackage{hyperref}       
\usepackage{url}            
\usepackage{booktabs}       
\usepackage{amsfonts}       
\usepackage{nicefrac}       
\usepackage{microtype}      
\usepackage{xcolor}         

\usepackage{amsthm}
\usepackage{amsmath}
\newtheorem{theorem}{Theorem}
\newtheorem{lemma}{Lemma}

\usepackage{optidef}

\usepackage{mwe}

\usepackage{tikz}
\usetikzlibrary{trees}

\usepackage{caption}
\usepackage{subcaption}

\usepackage{algorithm}
\usepackage{algpseudocode}
\usepackage{natbib}

\usepackage{enumitem}
\setitemize{noitemsep,topsep=0pt,parsep=0pt,partopsep=0pt}

\usepackage{nicefrac}       
\usepackage{microtype}      
\usepackage{fancyhdr}       
\usepackage{graphicx}       

\title{Cost-aware Generalized $\alpha$-investing for Multiple Hypothesis Testing}

\author{
  Thomas Cook \\
  Dept. of Mathematics \& Statistics \\
  University of Massachusetts \\
   Amherst, MA 01002\\
  \texttt{tjcook@umass.edu} \\
   \And
   Harsh Vardhan Dubey \\
  Dept. of Mathematics \& Statistics \\
  University of Massachusetts \\
   Amherst, MA 01002\\
  \texttt{hdubey@umass.edu} \\
   \And
  Ji Ah Lee \\
  Dept. of Mathematics \& Statistics \\
  University of Massachusetts \\
   Amherst, MA 01002\\
 \texttt{jiahlee@umass.edu} \\
   \And
  Guangyu Zhu \\
  Dept. of Statistics \\
  University of Rhode Island \\
  Kingston, RI 02881\\
  \texttt{guangyuzhu@uri.edu} \\
   \And
  Tingting Zhao \\
  Dept. of Information Systems \& Analytics \\
  Bryant University \\
  Smithfield, RI 02917\\
  \texttt{tzhao@bryant.edu} \\
   \And
  Patrick Flaherty\thanks{Corresponding Author} \\
  Dept. of Mathematics \& Statistics \\
  University of Massachusetts \\
   Amherst, MA 01002\\
  \texttt{pflaherty@umass.edu} \\
}

\begin{document}

\maketitle

\begin{abstract}
We consider the problem of sequential multiple hypothesis testing with nontrivial data collection costs.
This problem appears, for example, when conducting biological experiments to identify differentially expressed genes of a disease process.
This work builds on the generalized $\alpha$-investing framework which enables control of the false discovery rate in a sequential testing setting. 
We make a theoretical analysis of the long term asymptotic behavior of $\alpha$-wealth which motivates a consideration of sample size in the $\alpha$-investing decision rule. 
Posing the testing process as a game with nature, we construct a decision rule that optimizes the expected $\alpha$-wealth reward (ERO) and provides an optimal sample size for each test.
Empirical results show that a cost-aware ERO decision rule correctly rejects more false null hypotheses than other methods for $n=1$ where $n$ is the sample size.
When the sample size is not fixed cost-aware ERO uses a prior on the null hypothesis to adaptively allocate of the sample budget to each test.
We extend cost-aware ERO investing to finite-horizon testing which enables the decision rule to allocate samples in a non-myopic manner.
Finally, empirical tests on real data sets from biological experiments show that cost-aware ERO balances the allocation of samples to an individual test against the allocation of samples across multiple tests.
\end{abstract}



\section{INTRODUCTION}\label{sec:intro}

Machine learning systems are increasingly used to make decisions in uncertain environments.
Decision-making can be viewed in the framework of hypothesis testing in that a decision is made as the result of a rejection of the null hypothesis \citep{arrow1949bayes,dickey1970weighted, blackwell1979theory,verdinelli1995computing,parmigiani2009decision,berger2013statistical}.
When multiple hypotheses are under consideration, a FDR control procedure provides a way to control the rate of erroneous rejections in a batch of hypotheses for small-scale data sets \citep{benjamini1995controlling,storey2002direct,storey2004strong,benjamini2006adaptive,zeisel2011fdr, liang2012adaptive}.
However, these procedures typically require the test statistics of \textit{all} of the hypotheses under consideration so that the p-values may be sorted and a set of hypotheses may be selected for rejection.
In many modern problems the test statistics for all the hypotheses may not be known simultaneously and standard FDR procedures do not work.

Online FDR methods have recently been developed to address the need for FDR control procedures that maintain control for a sequence of tests when the test statistics are not all known at one time.
\citet{tukey1994collected} proposed the idea that one starts with a fixed amount ``$\alpha$-wealth'' and for each hypothesis under consideration, the researcher may choose to spend some of that wealth until it is all gone.
\citet{foster2008alpha} extended $\alpha$-spending by allowing some return on the expenditure of $\alpha$-wealth if the hypothesis is successfully rejected.
\citet{aharoni2014generalized} introduced generalized $\alpha$-investing and provided a deterministic decision rule to optimally set the $\alpha$-level for each test given the history of test outcomes.
A full review of related work is in Section~\ref{sec:related-work}. 

\subsection{Contributions} We extend generalized $\alpha$-investing to address the problem of online FDR control where the cost of data is not negligible.
Our specific contributions are:
\begin{itemize}
\item a theoretical analysis of the long term asymptotic behavior of $\alpha$-wealth in an $\alpha$-investing procedure,
\item a generalized $\alpha$-investing procedure for sequential testing that simultaneously optimizes sample size and $\alpha$-level using game-theoretic principles,
\item a non-myopic $\alpha$-investing procedure that maximizes the expected reward over a finite horizon of tests.
\end{itemize}

\subsection{Related Work}
\label{sec:related-work}
Tukey proposed the notion of $\alpha$-wealth to control the family-wise error rate for a sequence of tests~\citep{tukey1991philosophy,tukey1994collected}. \citet{foster2008alpha} proposed $\alpha$-investing, an online procedure that controls the marginal FDR (mFDR) for any stopping time in the testing sequence.
\citet{aharoni2014generalized} introduced generalized $\alpha$-investing  and provided a deterministic decision rule to maximize the expected reward for the next test in the sequence.
Recently, there has been much work on online FDR control in the context of A/B testing, directed acyclic graphs  and quality-preserving databases \citep{yang2017framework,ramdas2019sequential}. \citet{javanmard2018online} first proved that generalized $\alpha$-investing controls FDR, not only mFDR under an online setting with an algorithm called LORD. \citet{ramdas2017online} proposed the LORD++ to improve the existing LORD. Recent work leverages contextual information in the data to improve the statistical power while controlling FDR offline \citep{xia2017neuralfdr} and online \citep{ chen2020contextual}. \citet{ramdas2018saffron} then proposed SAFFRON, which also belongs to the $\alpha$-investing framework but adaptively estimates the proportion of the true nulls. 
All the aforementioned methods are synchronous, which means that each test can only start once the previous test has finished. \citet{zrnic2021asynchronous} extend $\alpha$-investing methods to an asynchronous setting where tests are allowed to overlap in time. These state-of-the-art online FDR control $\alpha$-investing methods do not address the needs for testing when the cost of data is not negligible. 
So, we propose a novel $\alpha$-investing method for a setting that takes into account the cost of data sample collection, the sample size choice, and prior beliefs about the probability of rejection.


Section~\ref{sec:alpha-investing} is a technical background of generalized $\alpha$-investing. 
Section~\ref{sec:theory} contains a theoretical analysis of the long term asymptotic behavior of the $\alpha$-wealth.
Section~\ref{sec:caero} presents a cost-aware generalized $\alpha$-investing decision rule based on the game-theoretic equalizing strategy.
Section~\ref{sec:synthetic-experiments} presents empirical experiments that show that the cost-aware ERO decision rule improves upon existing procedures when data collection costs are nontrivial. 
Section~\ref{sec:real-experiments} presents analysis of two real data sets from gene expression studies that shows cost-aware $\alpha$-investing aligns with the overall objectives of the application setting.
Finally, Section~\ref{sec:discussion} describes limitations and future work.

\section{BACKGROUND ON GENERALIZED $\alpha$-INVESTING}
\label{sec:alpha-investing}

Following the notation of \citet{foster2008alpha}, consider $m$ null hypotheses, $H_1, \ldots, H_m$ where $H_j \subset \Theta_j$.
The random variable $R_j \in \{0,1\}$ is an indicator of whether $H_j$ is rejected regardless of whether the null is true or not.
The random variable $V_j \in \{0,1\}$ indicates whether the test $H_j$ is both true and rejected.
These variables are aggregated as $R(m) = \sum_{j=1}^m R_j$ and $V(m) = \sum_{j=1}^m V_j$.
With these definitions, the FDR~\citep{benjamini1995controlling} is 
\begin{equation*}
    \text{FDR}(m) = P_\theta (R(m) > 0)\  \mathbb{E}_\theta \left[ \frac{V(m)}{R(m)} \mid R(m) > 0 \right],
\end{equation*}
and the marginal false discovery rate (mFDR) is
\begin{equation*}
    \text{mFDR}_\eta (m) = \frac{ \mathbb{E}_\theta \left[ V(m) \right] }{ \mathbb{E}_\theta \left[ R(m) + \eta \right] }.
\end{equation*}
Setting $\eta = 1-\alpha$ provides weak control over the family-wise error rate at level $\alpha$.

\citet{aharoni2014generalized} make two assumptions in their development of generalized $\alpha$-investing:
\begin{eqnarray}
  \label{eqn:assumption1}
  \forall \theta_j \in H_j : P_{\theta_j}(R_j | R_{j-1}, \ldots, R_1) &\leq& \alpha_j, \\
  \label{eqn:assumption2}
  \forall \theta_j \not\in H_j : P_{\theta_j}(R_j | R_{j-1}, \ldots, R_1) &\leq& \rho_j,
\end{eqnarray}
where
\begin{equation}
  \label{eqn:bestpower}
\rho_j = \sup_{\theta_j \in \Theta_j - H_j} P_{\theta_j}(R_j=1).
\end{equation}
Assumption~\ref{eqn:assumption1} constrains the false positive rate to the level of the test and Assumption~\ref{eqn:assumption2} is an upper bound of $\rho_j$ on the power of the test.
A pool of $\alpha$-wealth, $W_{\alpha}(j)$, is available to spend on the $j$-th hypothesis.
The $\alpha$-wealth is updated according to the following equations:
\begin{eqnarray}
  \label{eqn:alpha-wealth-update}
  W_{\alpha}(0) &=& \alpha \eta, \\
  W_{\alpha}(j) &=& W_{\alpha}(j-1) - \varphi_j + R_j \psi_j.
\end{eqnarray}
A deterministic function $\mathcal{I}_{W_{\alpha}(0)}$ is an $\alpha$-investing rule that determines: the cost of conducting the $j$-th hypothesis test, $\varphi_j$; the reward for a successful rejection, $\psi_j$; and the level of the test, $\alpha_j$:
\begin{equation}
  \label{eqn:alpharule}
  (\varphi_j, \alpha_j, \psi_j) = \mathcal{I}_{W_{\alpha}(0)}(\{R_1, \ldots, R_{j-1}\}).
\end{equation}
The $\alpha$-investing rule depends only on the outcomes of the previous hypothesis tests.
The Foster-Stine cost depends hyperbolically on the level of the test $\varphi_j = \alpha_j/(1-\alpha_j)$.

Generalized $\alpha$-investing can be viewed in a game-theoretic framework where the outcome of the test (reject or fail-to-reject) is random and the procedure provides the optimal amount of ``ante'' to offer to play and ``payoff'' to demand should the test successfully reject.
We make use of this game theoretic interpretation in our contributions in Section~\ref{sec:varphi}.

\citet{aharoni2014generalized} derive a linear constraint on the reward $\psi_j$ to ensure that, for a given $(\varphi_j, \alpha_j)$, the mFDR is controlled at a level $\alpha$ by ensuring the sequence $A(j)=\alpha R_j - V_j+ \alpha\eta - W_\alpha(j)$ is a submartingale with respect to $R_j$. 
Note that this constraint is not on the $\alpha$-wealth process directly.
Their constraint is 
\begin{equation}
    \label{eqn:mfdr_control}
    \psi_j \leq \min \left( \frac{\varphi_j}{\rho_j} + \alpha, \frac{\varphi_j}{\alpha_j} + \alpha - 1\right).
\end{equation}
Maximizing the expected reward of the next hypothesis test, $\mathbb{E}(R_j)\psi_j$, leads to the following equality
\begin{equation}
    \label{eqn:ero}
    \frac{\varphi_j}{\rho_j} = \frac{\varphi_j}{\alpha_j} - 1.
\end{equation}
Note that this equality selects the point of intersection of the two parts of the constraint in \eqref{eqn:mfdr_control}.
ERO $\alpha$-investing provides two equations for three unknowns in the deterministic decision rule.
\citet{aharoni2014generalized} address this indeterminacy by considering three allocation schemes for $\varphi_j$: constant, relative, and relative200 and suggest that the investigator can explore various options and set $\varphi_j$ on their own. Further details on these schemes are given in Section \ref{sec:synthetic-experiments}.

Since the dominant paradigm in testing of biological hypotheses is a bounded finite range for $\Theta_j$, for the remainder we assume $\Theta_j = [0, \bar{\theta}_j]$ for some upper bound, $\bar{\theta}_j$, and $H_j = \{0\}$.
This scenario may be viewed as a test that the expression for gene $j$ is differentially increased in an experimental condition compared to a control.
We consider a simple z-test here for concreteness.
The power of a one-sided z-test is $(1-\beta) := 1 - \Phi \left( z_{1-\alpha} + \frac{(\mu_0-\mu_1)}{\sigma /  \sqrt{n}} \right)$
where $z_{1-\alpha} = \Phi^{-1}(1-\alpha)$ is the z-score corresponding to level $\alpha$, $\mu_0$ is the expected value of the simple null hypothesis, $\mu_1$ is the expected value of the simple alternative hypothesis, $\sigma$ is the standard deviation of the measurements, and $n$ is the sample size.

Using Equation~\eqref{eqn:bestpower}, the best power under the previously defined $\Theta_j$ is
\begin{equation}
  \label{eqn:ztestbestpower}
  \rho_j = 1 - \Phi \left( z_{1-\alpha_j} - \frac{\bar{\theta}_j }{\sigma_j/ \sqrt{n_j}} \right).
\end{equation}
The best power depends on: (1) the level of the test, $\alpha_j$, (2) the scale of the bound on the alternative, $\bar{\theta}_j$, (3) the sample size, $n_j$, and the measurement standard deviation, $\sigma_j$.
One may compare multiple measurement technologies based on their precision by exploring the effect of changing $\sigma_j$ --- for example, for a fixed budget and all other things equal, a trade-off can be computed between more samples with a higher variance technology, versus fewer samples with a lower variance technology.
For the remainder, we assume $\sigma_j$ is fixed and known.
ERO $\alpha$-investing for Neyman-Pearson testing problems is solved by the following nonlinear optimization problem:
\begin{maxi!}|s|
{\alpha_j, \psi_j}
{\mathbb{E}_{\theta} (R_j) \psi_j}
{\label{eqn:np-ero-alpha-opt}}
{}
\addConstraint{\psi_j}{\leq \frac{\varphi_j}{\rho_j} + \alpha\label{eqn:opt1-fdrconst1}}
\addConstraint{\psi_j}{\leq \frac{\varphi_j}{\alpha_j} + \alpha - 1 \label{eqn:opt1-fdrconst2}}
\addConstraint{\frac{\varphi_j}{\rho_j}}{= \frac{\varphi_j}{\alpha_j} - 1 \label{eqn:opt1-eroconst}}
\addConstraint{\rho_j}{= 1 - \Phi \left( z_{1-\alpha_j} - \frac{\bar{\theta}_j }{\sigma_j/ \sqrt{n_j}} \right)\label{eqn:opt1-powerconst}}
\end{maxi!}
Constraints~\ref{eqn:opt1-fdrconst1} and \ref{eqn:opt1-fdrconst2} correspond to \eqref{eqn:mfdr_control} which controls the mFDR level, and constraint~\eqref{eqn:opt1-eroconst} ensures the maximal expected reward for the $j$-th test.
The optimal ERO solution still depends on an external choice of the sample size $n_j$, and the cost of the test $\varphi_j$.

\section{LONG-TERM $\alpha$-WEALTH}
\label{sec:theory}

Since the levels of future tests depend on the amount of $\alpha$-wealth available at the time of the tests, a theoretical consideration in generalized $\alpha$-investing is whether the long-term $\alpha$-wealth is submartingale or supermartingale (stochastically non-decreasing or stochastically non-increasing) for a given decision-rule. 
Most prior works include some implicit consideration of the behavior of the long-term $\alpha$-wealth.
In \citet{zhou}, which predates the seminal work of \citet{foster2008alpha}, test levels are set such that testing may continue indefinitely, even in the worst case scenario of no rejections, while still utilizing all initial $\alpha$-wealth. 
\citet{foster2008alpha} discuss strategies for setting the level of the test and provide some examples designed to accumulate $\alpha$-wealth for future tests. 
They also discuss the practical and ethical concerns with sorting easily rejected tests so as to accumulate an arbitrary amount of  $\alpha$-wealth before conducting more uncertain tests. 
\citet{aharoni2014generalized} seek to optimize the expected reward of the current test in an effort to maximize the $\alpha$-wealth available, and, in-turn, the levels for future tests.
\citet{javanmard2018online} discuss setting the vector $\gamma$ such that the power is maximimized for a mixture model set a-priori for the hypothesis stream.
In all of these methods, the motivation is to have sufficient $\alpha$-wealth to conduct tests,with an appreciable power perpetually.
Here we outline two scenarios where the long-term $\alpha$-wealth can be either submartingale or  supermartingale. 

In order to state the theorems regarding the $\alpha$-wealth sequence, we require a lemma bounding $\alpha$-wealth as a function of the prior probability of the null hypothesis.

\begin{lemma}
	\label{lemma:fs-alpha-bound}
	Given an $\alpha_j$-level for the $j$-th hypothesis test from rule $\mathcal{I}(R_1, \ldots R_{j-1})$, the expected value of $\alpha$-wealth for Foster-Stine $\alpha$-investing is
	\begin{equation}
		\label{eqn:fs-bound}
		\begin{split}
		\mathbb{E}^{j-1}\left[W_{j}\right] &\leq \\  - \frac{\alpha_j}{1-\alpha_j} &+ \left[ \rho_j - (\rho_j - \alpha_j) q_{j} \right] \left( \alpha + \frac{\alpha_j}{1-\alpha_j} \right),
		\end{split}
	\end{equation}
	where $\mathbb{E}^{j-1}\left[W_j\right] = \mathbb{E}\left[W(j) -W(j-1) | W(j-1)\right]$, and $q_{j} = \Pr[\theta_j \in H_j]$, the prior probability (belief) that the $j$-th null hypothesis is true. In the case of a simple null and alternative $\Theta_j=\{0,\bar{\theta}_j\}$, the bound is tight.
\end{lemma}
\begin{proof}
    Proof is provided in Appendix~\ref{sec:proof}.
\end{proof}
We are now in a position to understand dynamical properties of the expected value of the sequence of $\alpha$-wealth, $\{W(j) : j \in \mathbb{N}\}$.
\begin{theorem}[Submartingale $\alpha$-Wealth]\label{theorem:long-term-alpha-wealth-scenario1}
Given a simple null and alternative $\Theta_j=\{0,\bar{\theta}_j\}$,   $\{W(j) : j \in \mathbb{N}\}$ is submartingale  (stochastically non-decreasing) with respect to $\{R_1,\dots, R_{j-1}\}$ if
	\begin{equation}\label{eq:nondecreasing1}
		\rho_j\geq\frac{\alpha_j/(1-\alpha_j)}{\alpha+\alpha_j/(1-\alpha_j)}\frac{1}{1-q_{j}}.
	\end{equation} 
\end{theorem}
\begin{proof}
    Proof is provided in Appendix~\ref{sec:proof}.
\end{proof}
Theorem~\ref{theorem:long-term-alpha-wealth-scenario1} shows that one will be able to conduct an infinite number of tests in the long-term if the power is close to one or the prior probability of the null is close to zero. 
This scenario may occur when the hypothesis stream contains a large proportion of true alternative hypotheses, or if the sample sizes of the individual tests are large.

\begin{theorem}[Supermartingale $\alpha$-Wealth]\label{theorem:long-term-alpha-wealth-scenario2}
For any null and alternative hypothesis, $\{W(j) : j \in \mathbb{N}\}$ is supermartingale (stochastically non-increasing) with respect to $\{R_1,\dots, R_{j-1}\}$ if
 	\begin{equation}\label{eq:nonincreasing1}
 \rho_j \leq \left(\frac{\alpha_j/(1-\alpha_j)}{\alpha+\alpha_j/(1-\alpha_j)}-q_{j}\right)\frac{1}{1-q_{j}}.
 \end{equation} 
\end{theorem}
\begin{proof}
    Proof is provided in Appendix~\ref{sec:proof}.
\end{proof}

Theorem~\ref{theorem:long-term-alpha-wealth-scenario2} shows that the generalized $\alpha$-investing testing procedure will end in a finite number of steps if the power of the test is close to zero or the prior probability of the null hypothesis is close to one. 
This scenario may occur when the hypothesis stream is made up of a large proportion of true null hypotheses, or if the sample sizes used for each test results in an under powered test.

These theorems provide general insights for understanding when the $\alpha$-wealth can be expected to be (stochastically) non-decreasing or non-increasing.
The non-decreasing $\alpha$-wealth sequences require that $\rho_j \uparrow 1$ for a fixed $\bar{\theta}_j$ which, in the case of a Gaussian, would require $\sigma_j/\sqrt{n} \rightarrow 0$ or $n \rightarrow \infty$.
So, the $\alpha$-wealth grows unbounded if the sample size is unbounded.
This theory in combination with the premise of non-trivial experiment costs motivates the need for methods for cost-aware $\alpha$-investing when the sample size is not fixed.

If the sequence of hypotheses is under the control of the investigator, a strategy they might employ is to select many hypotheses that are likely to be rejected early so that a large amount of $\alpha$-wealth can be accumulated and then spent later on hypotheses that are less likely to be rejected, but are still of interest to the investigator.
This phenomenon is generally called piggybacking.
However, the issue with this strategy is that an investigator behaves differently if the difficult hypothesis is the first in the sequence of tests versus if the difficult hypothesis presents after a long sequence of easy tests. 
If the sequence of tests is not under the control of the investigator, they may still find themselves in a similar situation merely by a fortunate random ordering of the tests.
In Figure~\ref{fig:piggy_backing} shows that in ERO investing one can accumulate a large amount of $\alpha$-wealth and distort the expenditure of wealth for difficult tests that are subsequent to easy tests.
ERO investing rejects several true nulls while cost-aware ERO (CAERO) (Section~\ref{sec:caero}) does not exhibit such aggressive testing behavior.
Our premise, in this work, is that the investigator should be indifferent, in expectation, as to the position of the difficult hypothesis in the sequence of tests.

\begin{figure*}
    \centering
    \includegraphics[width = \textwidth]{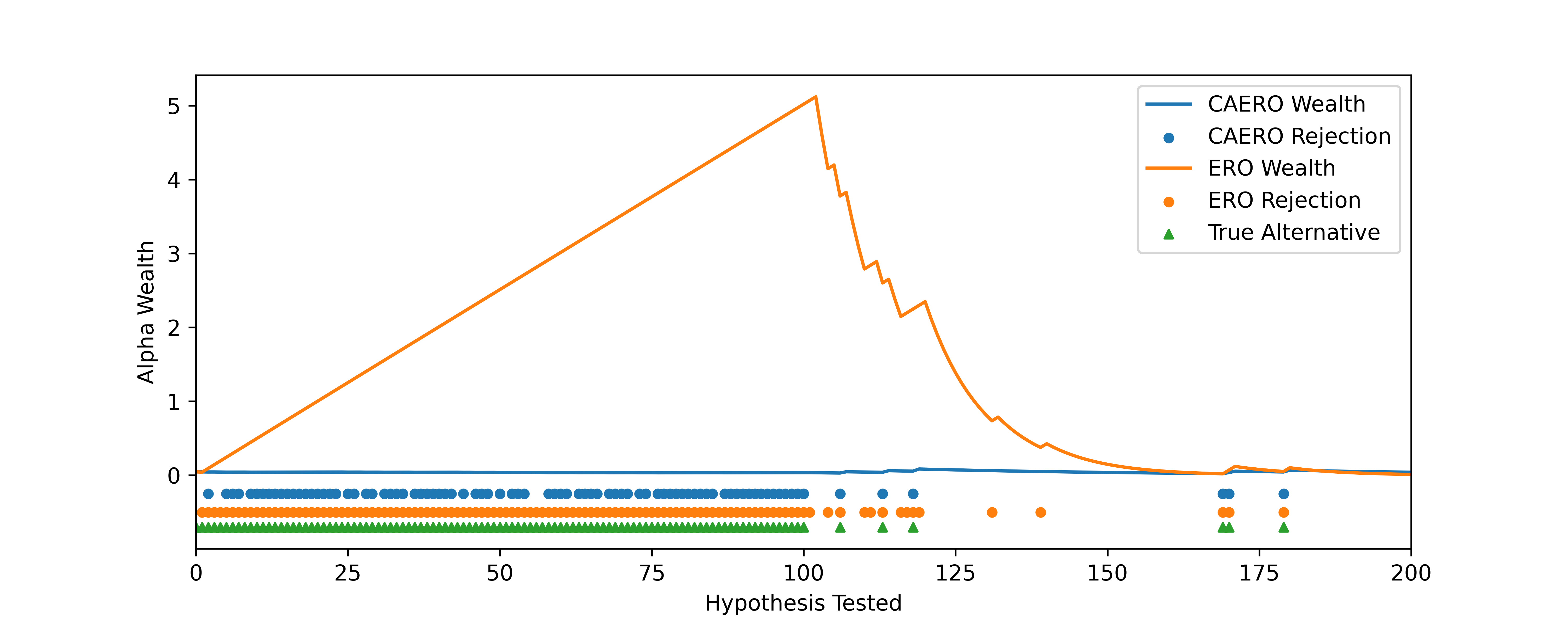}
    \caption{An example of piggybacking in an individual experiment where $q_j = 0.01$ for the first 100 hypotheses, and $q_j = 0.95$ for the remaining $100$ hypotheses. ERO investing with a relative spending scheme makes 8 false rejections following change in $q_j$, while CAERO (Section~\ref{sec:caero}) does not.}
    \label{fig:piggy_backing}
\end{figure*}

\section{COST-AWARE GENERALIZED $\alpha$-INVESTING}
\label{sec:caero}
      

In this section, our development derives from two key differences in assumptions compared to previous work.
First, the $\alpha$-cost of a hypothesis test, $\varphi_j$, should account for the a-priori probability that the null hypothesis is true as well as the pattern of previous rejections. 
The value of  $\varphi_j$ dictates bounds on $\alpha_j$ and  $\rho_j$, which, in combination with $q_j$, directly impact the behavior of $W_\alpha$. 
Consequently, setting $\varphi_j$ has an impact on the level of future tests. 
In Section~\ref{sec:varphi}, we extend the ERO problem from \citet{aharoni2014generalized} to include $\varphi_j$ in the optimization problem. 
Following this, we assume that the per sample monetary cost to conduct hypothesis tests is not trivial, motivating the need to include the sample size of a test, $n_j$, in the optimization. 
This extension is detailed in Section~\ref{sec:include-n}. 
In Section~\ref{sec:caero}, we present these extensions as a single decision rule, and extend this rule to a finite horizon in Section~\ref{sec:finite-horizon}.

\subsection{Optimizing $\varphi_j$}\label{sec:varphi}

A key contribution of this work is a procedure for selecting the amount of $\alpha$-wealth to commit to a given hypothesis test, $\varphi_j$.
\citet{aharoni2014generalized} leave $\varphi_j$ up to the investigator and provide several ways of selecting it: constant, relative, and relative-200.
Given a value of $\varphi_j$, the variables $\alpha_j$, $\rho_j$, and $\psi_j$ are chosen such that the expected reward, $\mathbb{E}_\theta(R_j)\psi_j$ is maximized. 
In their simulation studies, the choice of $\varphi_j$ is such that under the data-generating process one is expected to see one true alternative hypothesis before $\alpha$-death.
This section develops a principled method of setting $\varphi_j$ via a strategy for a two-player zero-sum game between the investigator and nature.

\begin{table}[t]
    \centering
\begin{tabular}{c c||c|c}
    \multicolumn{2}{c}{ } & \multicolumn{2}{c}{Player II (Nature)} \\ 
      &  & $\theta_j \in H_j$ & $\theta_j \not \in H_j$ \\
         \hline
         \hline
       Player I & Conduct Test &  $-\varphi_j + \alpha_j\psi_j$ & $-\varphi_j + \rho_j\psi_j$ \\ \cline{2-4}
        (investigator) & Skip Test & 0 & 0
    \end{tabular} \\
      \caption{Payoff matrix for posing hypothesis testing as a game against nature. The payouts shown are the expected value of the payout for a given pair of pure strategies.}
    \label{tab:payoff}
\end{table}
Suppose that we have a zero-sum game involving two players: the investigator (Player I) and nature (Player II). The investigator has two strategies - to test or to not test a hypothesis.
Nature, independent of the investigator, chooses to hide $\theta_j \in H_j$ with probability $q_j$ and $\theta_j \not \in H_j$ otherwise.
The utility function for this game is the change in $\alpha$-wealth.
The payoff matrix for the game is provided in Table \ref{tab:payoff}.

If the investigator chooses not to conduct the test, there is no cost ($\varphi_j=0$) and there is no reward ($\psi_j=0$) regardless of what nature has chosen.
So, the change in $\alpha$-wealth when not conducting a test is zero.
If the investigator chooses to conduct a test, and nature has hidden $\theta_j \in H_j$, then the payout is $-\varphi_j$ with probability $1-\alpha_j$, or $-\varphi_j + \psi_j$ with probability $\alpha_j$. 
In expectation, this payout is $-\varphi_j + \alpha_j \psi_j$. 
Similarly, if the investigator chooses to conduct the test and nature has hidden $\theta_j \not \in H_j$, then the expected payout is $-\varphi_j + \rho_j \psi_j$. 
The mixed strategy of nature is known to be $(q_j, 1 - q_j)$. 
What remains to be determined in this game are the unknowns in the payoff matrix, as well as the investigator's strategy. 
We choose to set the payoffs such that the expected change in $\alpha$-wealth is identical for both of the investigator's strategies. 
By designing the payoff matrix conditioned on nature's mixed strategy, such that the investigator has the same expected payoff for both pure (and any mixed) strategies, the investigator's choice to test or not test a hypothesis has no effect (in expectation) on the ability to perform future tests.
With these properties, nature is employing an \emph{equalizing} strategy.

The result is the investigator is indifferent as to whether to test or not test and the expected payoff is
\begin{equation} \label{eq:martingale}
q_j(-\varphi_j + \alpha_j \psi_j) + (1-q_j) (-\varphi_j + \rho_j \psi_j) = 0.
\end{equation}
Since the expected payoff for not testing is 0, this equation ensures that the expected change in $\alpha$-wealth when testing is also 0. By definition, this gives a self contained decision rule that makes $\alpha$-wealth martingale, striking a balance between the two scenarios given in Theorem \ref{theorem:long-term-alpha-wealth-scenario1} and Theorem \ref{theorem:long-term-alpha-wealth-scenario2}. 

\begin{theorem}[Martingale $\alpha$-Wealth]\label{theorem:long-term-alpha-wealth-martingale}
Given a simple null and alternative $\Theta_j=\{0,\bar{\theta}_j\}$,   $\{W(j) : j \in \mathbb{N}\}$ is martingale with respect to $\{R_1,\dots, R_{j-1}\}$ if
	\begin{equation}\label{eq:martingale_thm}
    \rho_j = \left( \frac{1}{1-q_j} \right) \left( \frac{\varphi_j}{\psi_j} - q_j \alpha_j\right).
	\end{equation} 
\end{theorem}
\begin{proof}
    Proof is provided in Appendix~\ref{sec:proof}.
\end{proof}
Theorem \ref{theorem:long-term-alpha-wealth-martingale} provides a condition on the power of the test which requires a balance between the ratio of $\varphi_j$ and $\psi_j$ and the probability of a false positive.

Allowing $\varphi_j$ to be a free variable in the optimization problem leads us to an infinite number of ERO-class solutions. 
Selecting the maximum of this class would lead to aggressive play, and in many situations leads to \emph{betting the farm} or \emph{bold play}. 
The martingale constraint, \eqref{eq:martingale}, reduces the solution space to a single nontrivial solution and a trivial solution where $\varphi_j = \psi_j = \alpha_j = 0$. 
Furthermore, the fact that the expected reward is constrained to be zero by \eqref{eq:martingale} means that the ERO optimization problem is a feasibility problem.
Even so, we retain the objective function in the optimization problem in anticipation for the next section where we allow for variable sample sizes $n_j$.

The investigator may additionally wish to limit the variance of their wealth process. 
This can be accomplished by setting an \emph{upper bound} on the proportion of wealth an investigator may spend for an ante.
Furthermore, the investigator may wish to \emph{lower bound} the power of an individual test.
In order to satisfy these conditions, the investigator will collect more samples for an individual hypothesis in order to meet their power requirement.
We choose to impose a lower bound constraint on the power for this reason.

\subsection{Optimizing $n_j$}\label{sec:include-n}
In the previous section $n_j$ was held fixed; we now consider $n_j$ as a free variable in our optimization problem. 
As a result, $\rho_j$ is no longer completely determined by $\alpha_j$ and one can increase $\rho_j$ via increasing $n_j$.
In scientific applications, the investigator is also constrained by sample collection costs, and would not wish to spend excessively on a single hypothesis test.
We modify the generalized $\alpha$-investing decision rule, \eqref{eqn:alpharule}, to include a notion of dollar-wealth $W_{\$}(j)$ available for expenditure to collect data to test the $j$-th hypothesis
\begin{equation}
    \label{eqn:augmented-decision-rule}
    (\varphi_j, \alpha_j, \psi_j, n_j) = \mathcal{I}(W_{\alpha}(0), W_{\$}(0)) ( \{ R_1, \ldots, R_{j-1} \}),
\end{equation}
where $n_j$ is the sample size allocated for testing of the $j$-th hypothesis.
A natural update for the dollar-wealth is
\begin{eqnarray}
    \label{eqn:dollar-wealth-update}
    W_{\$}(0) &=& B \\
    W_{\$}(j) &=& W_{\$}(j-1) - c_j n_j,
\end{eqnarray}
where $c_j$ is the per-sample cost for data to test the $j$-th hypothesis, and $B$ is the initial dollar-wealth.
Allowing the cost to vary with the hypothesis test enables one to model different experimental methods and cost inflation for long-term experimental plans.

Since $n_j$ is a free variable under the control of the investigator, we again have many solutions to the ERO problem.
Furthermore, as $n_j$ increases up to the allowable expenditure of sample resources, so does the power and therefore the expected reward.
Theoretically, the optimal solution allocates all the sample budget available even though the marginal increase in power and thus the expected reward, is vanishingly small for large $n_j$. 
To address this issue, we modify the objective function to include a small penalty for increasing sample size\footnote{We thank an anonymous reviewer for this suggestion.}, $-\lambda c_j n_j $ where $\lambda$ is chosen by the investigator.
The solution is fairly robust to the value of $\lambda$ and this reformulation provides for a unique optimal value.

\subsection{Cost-aware ERO Decision Rule}
The investigator's goal is to conduct as many tests as possible while rejecting as many true alternatives as possible and maintaining control of the false discovery rate.
Incorporating the methods for optimizing $\varphi_j$ and $n_j$ into the ERO problem yields a self-contained decision rule in the form of \eqref{eqn:augmented-decision-rule},
\begin{maxi!}|s|
{\varphi_j, \alpha_j, \psi_j, n_j}
{\mathbb{E}_{\theta}(R_j)\psi_j - \lambda c_j n_j}
{\label{eqn:opt2}}
{}
\addConstraint{\psi_j}{\leq \frac{\varphi_j}{\rho_j} + \alpha\label{eqn:opt2-fdrconst1}}
\addConstraint{\psi_j}{\leq \frac{\varphi_j}{\alpha_j} + \alpha - 1 \label{eqn:opt2-fdrconst2}}
\addConstraint{\frac{\varphi_j}{\rho_j}}{= \frac{\varphi_j}{\alpha_j} - 1 \label{eqn:opt2-eroconst}}
\addConstraint{\rho_j}{= 1 - \Phi \left( z_{1-\alpha_j} - \frac{\bar{\theta}_j }{\sigma_j/ \sqrt{n_j}} \right)\label{eqn:opt2-powerconst}}
\addConstraint{\rho_j}{\geq \rho_{\text{lo}}\label{eqn:opt2-powerconst_min}}
\addConstraint{\varphi_j}{\leq a  W_{\alpha}(j-1)\label{eqn:opt2-alphawealthconst}}
\addConstraint{n_j c_j}{\leq W_{\$}(j)\label{eqn:opt2-dollarconst}}
\addConstraint{0}{= \mathbb{E}\left[\Delta W_{\alpha}\right], \label{eqn:opt2-indifference}}
\end{maxi!}

Constraints~\eqref{eqn:opt2-fdrconst1} and \eqref{eqn:opt2-fdrconst2} ensure control over the mFDR.
Constraint~\eqref{eqn:opt2-powerconst} connects the level, power, and sample size of the test.
Constraints~\eqref{eqn:opt2-alphawealthconst} and \eqref{eqn:opt2-dollarconst} ensure the existing $(\alpha, \$)$-wealth is not exceeded. 
The parameter $a \in (0,1]$ controls the proportion of $\alpha$-wealth that a single test can be allocated.
Constraint~\eqref{eqn:opt2-indifference} ensures nature's strategy is equalizing. Written out explicitly,
$$\mathbb{E}\left[\Delta W_{\alpha}\right] = (-\varphi_j + \psi_j) \Pr[R_j=1] + (-\varphi_j) \Pr[R_j=0],$$
$$ \Pr[R_j=1] = \alpha_jq_j + \rho_j(1-q_j), $$
$$ \Pr[R_j=0] = (1-\alpha_j)q_j + (1-\rho_j)(1-q_j).$$
A pseudo-code algorithm of the full cost-aware ERO method and further extensions to cost-aware ERO are described in Appendix~\ref{sec:extensions}. 

\begin{lemma}
    The cost-aware ERO decision rule ensures that $\alpha$-wealth, $W_\alpha(j)$, is martingale with respect to $\{R_1,\dots, R_{j-1}\}$. 
\end{lemma}
\begin{proof}
    Constraint~\eqref{eqn:opt2-indifference} sets $\alpha$-wealth to be martingale by definition.
\end{proof}
Constraint~\eqref{eqn:opt2-indifference} sets the expected change in $\alpha$-wealth equal to zero. 
This enforces that $W_{\alpha}(j)$ is martingale.
Allowing $W_{\alpha}(j)$ to be submartingale, as per Theorem \ref{theorem:long-term-alpha-wealth-scenario1}, can lead to a situation where hypotheses are tested at high $\alpha$-levels due to the accumulated $W_\alpha$ from previous rejections. This is referred to as piggybacking in the literature when such accumulated wealth leads to poor decisions \citep{ramdas2017online}.
On the other hand, allowing $W_\alpha(j)$ to be supermartingale, as per Theorem \ref{theorem:long-term-alpha-wealth-scenario2}, causes the testing to end, and is referred to as $\alpha$-death in the literature.
Using a game-theoretic formulation allows us to propose an expected-reward optimal procedure which considers preventing $\alpha$-death and piggy-backing. 

Constraint~\eqref{eqn:opt2-indifference} only controls the \emph{expected} increment in $W_{\alpha}$. 
It is well known that martingale-based strategies can suffer from what is known as \emph{gambler's ruin}. 
Since no bounds are set on the worst case scenario, which in this case is when $R_j = 0$, it is possible that we could set $\varphi_j = W_{\alpha}(j-1)$, and suffer $\alpha$-death. 
This occurs, for example, when $q_j$ is close to $0$, and $\mu_A$ and $\sigma_j$ allow for $\rho_j \to 1$.
In such a case, a rejection is almost certain, and in turn, so is receiving the reward $\psi_j$. 
Recall that we restrict ourselves to an ERO solution, and thus, we can interpret constraint~\eqref{eqn:opt2-indifference}, without the factor $a$, as setting $\varphi_j$ to the expected reward - the quantity that we are maximizing.
In order to keep $W_{\alpha}$ martingale, this almost guaranteed upside must be counteracted by a devastating downside.
In order to avoid $\alpha$-death, we add a factor which limits the maximal bet, preventing the investigator from \emph{betting the farm}. 
In simulation studies, we found that setting $a = 0.025$ gave good results. We require $\rho_j \geq 0.9$ when $n$ is not constrained.  

\subsection{Finite horizon cost-aware ERO $\alpha$-investing}\label{sec:finite-horizon}

\begin{figure*}[t]
  \centering
    \includegraphics[width=0.9\textwidth]{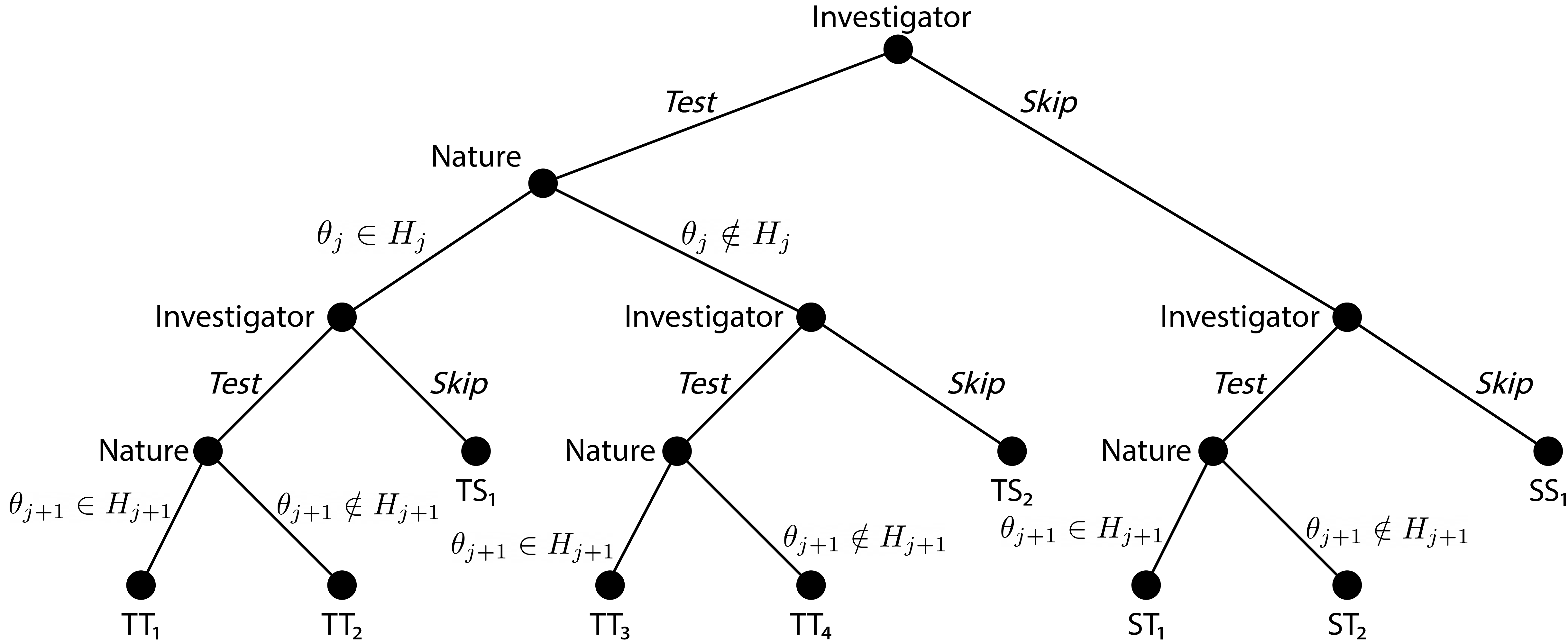}
  \caption{Extensive form of two-step game between Investigator (Player I) and the Nature (Player II). Strategies for each player are italicized. The leaves are labeled to denote the strategy taken by the investigator and are enumerated for equations presented in Appendix~\ref{sec:two-step}.}
  \label{fig:two-step-extensive}
\end{figure*}

The standard ERO framework optimizes only the one-step expected return, $\mathbb{E}_\theta (R_j)\psi_j$.
But, when tests are expensive, it is logical to consider the expected return after two (or more) tests.
We consider $q_{j}$ to be known, and extend the game theoretic framework to a finite horizon of decisions.
The extensive form of the game between nature, who hides $\theta_j$ in the null or alternative hypothesis region, and the investigator, who seeks to find $\theta_j$ and gain the reward for doing so, is shown in Figure \ref{fig:two-step-extensive}.
We note that sequential two-step cost-aware ERO investing is a different problem than batch ERO investing because two-step investing accounts for the expected change in $(\alpha,\$)$-wealth after each step while batch cost-aware ERO only received the payoff at the conclusion of all of the tests in the batch.

The two-step objective function is
\begin{equation}
\begin{aligned}
    \mathbb{E} ( R_j\psi_j &+ R_{j+1} \psi_{j+1} ) = \mathbb{E}(R_j)\psi_j +
     \psi_{j+1} [ P(R_j=0)\mathbb{E}(R_{j+1}|R_j = 0) + P(R_j=1)\mathbb{E}(R_{j+1}|R_j = 1)]
    \end{aligned}
\end{equation}
with constraints~\eqref{eqn:opt2-fdrconst1}-\eqref{eqn:opt2-indifference} from Problem~\ref{eqn:opt2} remaining for steps $j$ and $j+1$.
Designing the game so that nature's strategy is an equalizing strategy results in a system of equations (Appendix~\ref{sec:two-step}) that form constraints in the ERO optimization problem. 
It is worth noting that such a game simplifies to the standard cost-aware ERO method defined in \ref{eqn:opt2} when $W_{\$} >>0$. 
This holds since the parameters of the second test depend on the expected $\alpha$-wealth available at that step. 
When the expected increment is $0$, as set in constraint~\eqref{eqn:opt2-indifference}, and when available $\$ $-wealth is not scarce, then each step is equivalent to optimization occurring on the first test.
When this constraint is lifted, or when the available $W_\$ $ is low, the finite horizon solution provides a different solution to the single step solution.

\section{SYNTHETIC DATA EXPERIMENTS}
\label{sec:synthetic-experiments}

\paragraph*{Experimental Settings} 
To compare our method with state-of-the-art related methods, we generate synthetic data as described in \citet{aharoni2014generalized}. 
The synthetic data is composed of $m=1000$ possible hypothesis tests.
For the $j$-th test, the true state of $\theta_j$ is set to the null value of $0$ with probability $q_j \sim \text{Unif}(0.85,0.95)$ and otherwise set to $2$.
A set of $n_j = 1000$ potential samples $(x_{ji})_{i=1}^{n_j}$ were generated i.i.d from a $\mathcal{N}(\theta_j,1)$ distribution.
For each hypothesis test, the z-score was computed as  $z_j = \sqrt{n_j^*} \sum_{i=1}^{n_j^*} x_{ji}$, where $n_j^*$ is described in the table and the one-sided p-value is computed.
The methods were tested on $10,000$ realizations of this simulation data generation mechanism. Pseudo-code, as well as other implementation details, for this simulation can be found in Appendix \ref{sec:simulation-setup}.

\subsection{Comparison to state-of-the-art methods }
Table~\ref{tbl:sota} compares our method, cost-aware ERO, with related state-of-the-art $\alpha$-investing methods including: $\alpha$-spending~\citep{tukey1994collected}, $\alpha$-investing~\citep{foster2008alpha}, $\alpha$-rewards~\citep{aharoni2014generalized}, ERO-investing~\citep{aharoni2014generalized}, LORD~\citep{javanmard2018online, ramdas2017online}, and SAFFRON~\citep{ramdas2018saffron}.
The table is indexed by the allocation scheme (Scheme), and the reward method (Method). 
The allocation scheme determines the value of $\varphi_j$ at each step, which in many cases is left to user discretion. 
We implement the $\varphi$-allocation schemes proposed by \citet{aharoni2014generalized}.
The constant scheme simply allocates, 
\begin{equation*}
    \varphi_j = \min \left\{\frac{1}{10} W_\alpha(0), W_\alpha(j-1)\right\},
\end{equation*} 
for each test, the relative scheme allocates an amount that is proportional to the remaining $\alpha$-wealth,
\begin{equation*}
    \varphi_j = \frac{1}{10}W_\alpha(j-1)
\end{equation*}
and continues until $W_\alpha(j) < (1/1000) W_\alpha(0)$. 
The relative 200 scheme follows the same proportional steps as the relative, but always performs 200 tests~\citep{aharoni2014generalized}.
The results from our implementation of these methods matches or exceeds previously reported results.

\begin{table*}[]
    \centering
    \begin{tabular}{llrrrr}
\toprule
          &         &    Tests &  True Rejects &  False Rejects &  mFDR \\
Scheme & Method &          &               &                &       \\
\midrule
constant & $\alpha$-spending &    10.0 &   0.28 &  0.04 &  0.033\\
          & $\alpha$-investing &    16.0 &   0.44 &  0.07 &  0.045 \\
          & $\alpha$-rewards $k = 1$ &    14.6 &   0.40 &  0.06 &  0.043 \\
          & $\alpha$-rewards $k = 1.1$ &    16.3 &   0.43 &  0.06 &  0.043 \\
          & ERO investing &    18.9 &   0.53 &  0.08 &  0.051 \\
\midrule
relative & $\alpha$-spending &   66.0 &	0.55 &	0.04 &	0.028 \\
          & $\alpha$-investing &    81.8 &	0.87 &	0.09 &	0.045 \\
          & $\alpha$-rewards $k = 1$ & 81.1 &	0.85 &	0.08 &	0.043\\
          & $\alpha$-rewards $k = 1.1$ &    80.8 &	0.82 &	0.08 &	0.041 \\
          & ERO investing &   83.2 &	0.93 &	0.90 &	0.045\\
\midrule          
other & LORD++ &  1000.0 &	2.06 &	0.07 &	0.022\\	
          & LORD1 & 1000.0 &	1.46 &	0.03 &	0.014 \\
          & LORD2 &  1000.0	& 1.97 &	0.06 &	0.020 \\
          & LORD3 &  1000.0	& 1.99 &	0.07 &	0.024  \\
          & SAFFRON &  1000.0 &	1.28 &	0.07 &	0.031 \\
\midrule
  cost-aware & ERO $n_j = 1$ &   953.0 &   \textbf{4.23} &  0.12 &  0.023 \\
\midrule
\midrule
cost-aware & ERO $n_j^*$ &  225.7 &	\textbf{19.11} &	0.22 &	0.011 \\
other & LORD++ (n = 3) &  334.0 &	22.54 &	0.79 &	0.032\\	
\bottomrule
\end{tabular}
    \caption{Comparison of cost-aware $\alpha$-investing with state-of-the-art sequential hypothesis testing methods. Values for Tests, True Rejects and False Rejects are the average across 10,000 iterations, and these estimates are used for mFDR. All methods are constrained to use 1000 samples at most per iteration. 
    For comparison include LORD++ with the optimal sample size of  $n = 3$. However, the optimal sample size for LORD++ was selected by observing the number of true rejects for $n_j \in [1,10]$ and this information would not be available to an investigator. The optimal value of $n_j^*$ for cost-aware ERO, however, was predictable from the observed data.}
    \label{tbl:sota}
\end{table*} 

ERO investing yields more true rejects than $\alpha$-spending, $\alpha$-investing, and both $\alpha$-rewards methods.
The LORD variants and SAFFRON perform the maximum number of tests while maintaining control of the mFDR.
For the use scenarios considered in the LORD and SAFFRON papers (large-scale A/B testing), this is optimal --- tests are nearly free and the goal is to be able to keep testing while maintaining mFDR control.
The cost-aware ERO setting is different and more applicable to biological experiments where one aims to maximize a limited budget of tests to achieve as many true rejects as possible while controlling the mFDR.
Increasing the sample size capacity for each test enables cost-aware ERO to achieve higher power with fewer tests than the current state-of-the-art methods with $n = 1$.
For fair comparison, we varied $n$ for LORD++ and include the sample size which maximized the number of true rejections. 
This selection was performed \emph{after} running all considered sample sizes. 
It is important to note that the investigator would not have access to such information in a real experiment.
Releasing the restriction on sample size enables cost-aware ERO to allocate an adaptive number of samples based on the prior of the null as well as the available budget.
Appendix~\ref{sec:sota_q1} shows the comparisons for $q = 0.1$ and Appendix~\ref{sec:large_n} shows comparisons with all of the other methods set to $n_j=10$. 
Our cost-aware ERO method with $n=1$ performs more tests and rejects more false null hypotheses than all competing methods at $n = 1$. 

\begin{figure*}[]
  \centering
    \includegraphics[width=0.9\textwidth]{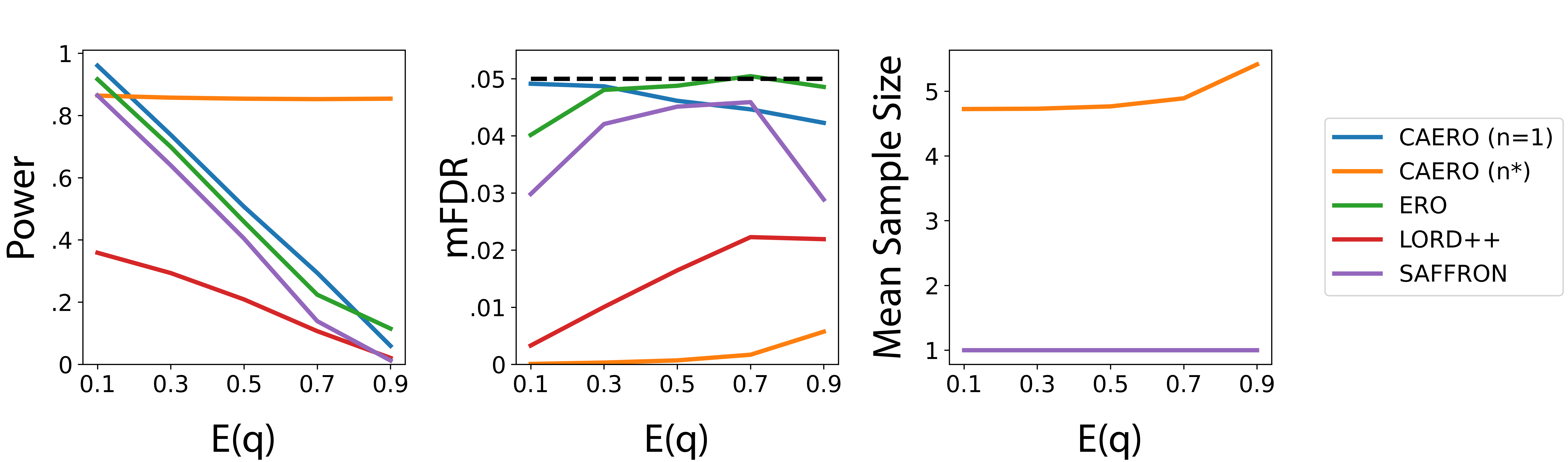}
  \caption{Power, mFDR, and mean number of samples per test for cost-aware ERO and existing methods ($n=1$) with random $q_{j} \sim \text{Beta}$.}
  \label{fig:simulation-results}
\end{figure*}

It is worth noting that ERO and cost-aware ERO with $n_j = 1$ are still quite different despite the restriction of sample size. 
We can view the difference in performance between these two methods as the benefit of allocating $\varphi_j$ using our game-theoretic framework.
Our decision rule incorporates our prior knowledge of the probability of the null hypothesis being true and aims to maintain $\alpha$-wealth (as a martingale). 
The experimental set up of \citet{aharoni2014generalized} implicitly leverages similar prior knowledge in the spending schemes proposed.
All spending schemes proposed in \citet{aharoni2014generalized} allow us to test at least one true alternative, in expectation, at which point the $\alpha$-wealth should increase. 
This increase in $\alpha$-wealth should then sustain testing until another true alternative appears. 
However, in the cost-aware ERO optimization problem, this information is explicitly accounted for, and helps us avoid situations described in Theorem \ref{theorem:long-term-alpha-wealth-scenario1} and Theorem \ref{theorem:long-term-alpha-wealth-scenario2}. 
By restricting $n_j = 1$, we have effectively limited our ability to inflate $\rho_j$ with a large sample size, and influence $W_\alpha (j)$ towards being submartingale. 
On the other hand, nature's equalizing strategy limits the expected payout to 0, by limiting the size of $\varphi_j$, preventing the experimenter from experiencing $\alpha$-death quickly, as seen in the constant spending scheme.

\subsection{Computation and Implementation}
In our experiments, for one set of $1,000$ potential hypothesis tests ERO investing, cost-aware, and finite-horizon cost-aware ERO all take $\sim 30$ seconds on a single 2.5GHz core and 16Gb RAM.
The nonlinear optimization problem was solved using CONOPT~\citep{drud1994conopt}.
Because the solver depends on initial values and heuristics to identify an initial feasible point, infrequently the solver was not able to find a local optimal solution; in these instances, the solver was restarted 10 times and if it failed on all restarts the iteration was discarded. 
Out of $10,000$ data sets at most $27$ iterations were discarded (for $n_j=1$).
Code to replicate these experiments is available at \url{https://github.com/ThomasCook1437/cost-aware-alpha-investing}.

\subsection{Random Prior of the Null Hypothesis}
\label{sec:random-q}

One of the benefits of incorporating a notion of sampling cost into the hypothesis testing problem is the ability to allocate resources based on the prior probability of the null, $q$.
We generated simulation data as previously described except the prior probability of the null hypothesis is selected at random from $q_{j} \sim \text{Beta}(a,b)$ where $a+b=100$ and with $2,500$ independent realizations of the data.
Appendix \ref{sec:simulation-setup} contains pseudo-code and further implementation details.
Figure~\ref{fig:simulation-results}(a-c) shows the power, mFDR, and mean number of samples per test as a function of $\mathbb{E}[q_{j}]$.
The results show that cost-aware ERO $\alpha$-investing achieves high power while maintaining control of the mFDR.
A key result of this experiment is that should it not be possible to collect as many samples as the optimization problem yields, the investigator may choose to not perform the test at all and instead wait for a test (with associated prior) that does yield an optimal sample size within the budget or may choose to allow the $\alpha$-wealth ante to adjust to the bound on the sample size.
This often occurs for large values of $q_j$, which we know by Theorem \ref{theorem:long-term-alpha-wealth-scenario2} will influence $W_\alpha(j)$ towards behaving as a supermartingale.
Cost-aware ERO will compensate by increasing $\rho_j$ through the sample size, $n_j$, and will expend the $W_{\$}$ available, as the optimization only considers a single step.
It is worth noting, that when $\mathbb{E}[q_{j}]$ is close to 1, cost-aware ERO with $n = 1$, maintains power better than other methods. 
This can be attributed to the allocation scheme that constraint~\eqref{eqn:opt2-indifference} creates. 
The value of $\varphi_j$ is kept small so that multiple false null hypotheses are tested at an appreciable level so that $\alpha$-wealth can be earned, and testing sustained. This setting is common in biological settings, where false null hypotheses can be rare.


\subsection{Sensitivity to the Prior}
Cost-aware ERO makes explicit the prior on the null, while the dependence on the probability of null hypotheses in the sequence is more implicit in other methods. 
So, an important question is, how sensitive is the method to misspecification of $q_j$.
To address this question, we consider two types of misspecification across the hypothesis sequence: variance with a correct expectation, and bias Appendix~\ref{sec:uncertainty-q}.
We find that cost-aware ERO is robust to variance in $q$ with a correct expectation. 
This is likely due to the the property that cost-aware ERO is relatively conservative in its allocation of $(\alpha, \$)$-wealth and the method has the opportunity to recover from losses due to misspecification.
However, the cost-aware ERO is sensitive to a biased specification of $q_j$.
If the true probability of the null is 0.9 on average and $q=0.85$ is used in cost-aware ERO, 273 fewer tests are performed compared to using the correct value of $q$.
Essentially, the downward bias in the assumed $q$ causes cost-aware ERO to be more aggressive in spending $\alpha$-wealth than it should be.
In practice, this effect could be mitigated, but ensuring that $\alpha$-wealth spending is conservative or by giving a margin of safety to the assumed value of $q_j$.
However, a more principled solution would employ a robust optimization formulation of cost-aware ERO or to implement online-learning for the $q$ process.
While this modification is outside of the scope of this paper, it is of great interest.

\subsection{Finite-Horizon Cost-aware ERO Investing}
To test whether extending the horizon of the reward to be maximized would enable better decisions as to $(\alpha, \$)$-wealth allocation, we varied the length of the horizon considered in the cost-aware ERO investing decision rules.
In general, the optimal values returned are identical between the decision rules. This is especially visible at the beginning of the testing process.
Discrepancies occur when $W_{\$}$ is sufficiently low such that repeatedly applying the one-step cost-aware ERO decision rule would expend all $W_{\$}$ prior to the final test in the finite horizon. 
This occurs when the finite-horizon is set to be a large number of steps or when the experiment is near the end of its funding.
We also noticed that our solver exhibited less stability as the length of the horizon increased.

\begin{figure*}[t]
  \centering
    \includegraphics[width=0.9\textwidth]{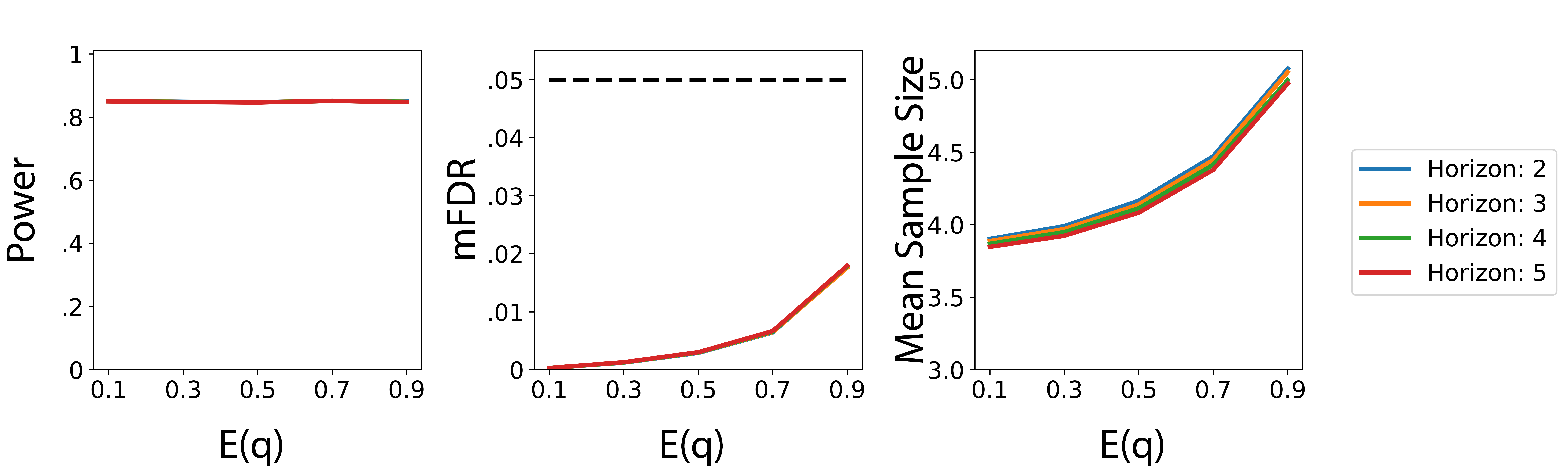}
  \caption{Power, mFDR, and mean number of samples per test for finite horizon cost-aware ERO with random $q_{j} \sim \text{Beta}$. A larger horizon corresponds to a greater number of future tests considered in the optimization process.}
  \label{fig:simulation-results-fh}
\end{figure*}

As seen in Figure \ref{fig:simulation-results-fh}, extending the horizon to include more tests results in the same allocation of samples. For $\mathbb{E}[q] = 0.9$, which we consider most applicable to biological applications, Table \ref{tbl:finite_horizon} confirms that the solutions for different horizons are identical, with discrepancies occurring due to computational constraints.

\begin{table*}[t]
    \centering
    \begin{tabular}{llrrrr}
\toprule
                   &    Tests &  True Rejects &  False Rejects &  mFDR \\
 Horizon &          &               &                &       \\
\midrule
 2 &    191.0 &   16.25 &  0.31 &  0.018\\
   3 &    187.0 &   15.90 &  0.30 &  0.018 \\
   4 &    181.0 &   15.37 &  0.30 &  0.018 \\
   5 &    177.3 &   15.04 &  0.29 &  0.018 \\

\bottomrule
\end{tabular}
    \caption{Varying the size of the finite horizon when $q_j \sim Beta(90,10)$. Values displayed correspond to the mean across 2,500 repetitions.}
    \label{tbl:finite_horizon}
\end{table*} 

 Note that this technique optimizes the parameters of each test based on the \emph{expected} wealth available at that time. The parameters set for future tests will never truly be attained. 
 These results demonstrate that our principled $(W_{\alpha}, W_{\$})$ spending strategy considering one step sufficiently captures the effect of the current test on our future tests. 
 The martingale constraint enables us to conduct tests so that future tests remain powerful, and we do not benefit from adding additional information to our optimization problem. 
 These results simultaneously suggest that an extended horizon may be appropriate for contexts where the optimization objective is not restricted to the expected reward or where the martingale constraint is not set for each individual test.

\section{REAL DATA EXPERIMENTS}
\label{sec:real-experiments}
Biological experiments are typically such that the sample costs are non-trivial, the proportion of false null hypotheses is small, and the number of overall tests is large. 
Our methods were compared to the ERO method on two gene expression data sets. 
The results show that the cost-aware method performs more tests and rejections, while spending fewer samples than a method which does not have the capability to adapt the sample size. 
As there is no ground-truth for these data sets, we are unable to compare the number of true rejections.

\subsection{Prostate Cancer Data}
Gene expression data was collected to investigate the molecular determinants of prostate cancer~\citep{dettling2004bagboosting}. 
The data set contains 50 normal samples and 52 tumor samples and each sample is a $m=6033$ vector of gene expression levels.
The data set has been normalized and log-transformed so that the data for each gene is roughly Gaussian. 
Let the empirical mean and standard deviation of the log-transformed normal samples be denoted $\hat{\mu}_j$ and $\hat{\sigma}_j$ respectively and let the log-transformed tumor data be denoted $x_j \in \mathbb{R}^{52}$.
The goal is to test whether the tumor gene expression is increased relative to the normal samples.

We use this data set to simulate a sequential testing scenario across genes. 
To estimate a prior for the null hypothesis for each gene, a logistic function was fit to only the first two samples, $\hat{q}_j = 1 - \left( 1 + \exp(-\beta ([\bar{x_j}]_{1:2} - x_0) )\right)^{-1}$, where $x_0 = \log_{10}(4)/\sigma$ and $\beta = 2$; these first two samples were then removed from the data set.
The order of the genes was permuted randomly and the cost-aware decision function was computed for each gene in sequence with $\hat{q}_j$ as described and $\bar{\theta}_j = \log_{10}(2)/\hat{\sigma}_j$.
We compared cost-aware ERO to ERO investing with the maximum number of samples available ($n=50$) and with a $n=3$ because a typical default replication level in biological experiments is to conduct experiments in triplicate.
For both procedures $c_j=1, \forall j$, and $W_{\$}[0] = 1000$. 
Pseudo-code and implementation details for this experiment can be found in Appendix \ref{sec:simulation-setup}.

\begin{figure*}[t]
\centering
    \includegraphics[width=\textwidth]{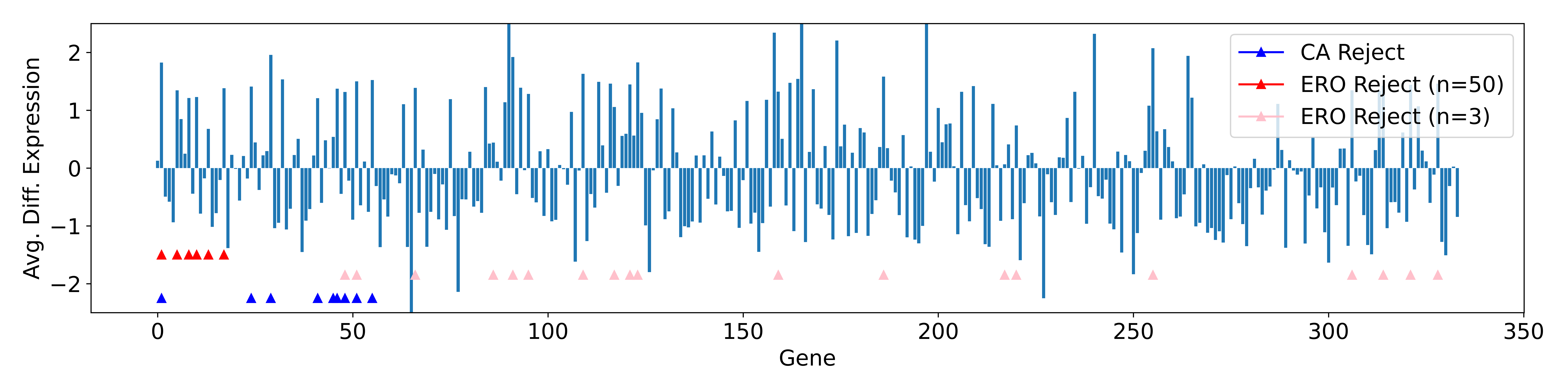}
\caption{Comparison of Cost-aware ERO investing and ERO investing for a prostate cancer gene expression data set. 
ERO ($n=50$) rejects many tests early, but suffers from an aggressive expenditure of sample collection resources and is unable to test beyond the 20th gene. 
ERO ($n=3$) is able to reject more tests, but fails to reject early hypotheses, observes a noisier measurement of the true differential gene expression, and benefits from a piggybacking effect for later tests.
Cost-aware ERO distributes the finite allocation of samples across a smaller set of genes than ERO ($n=3$), but a larger set of genes than ERO ($n=50$). 
It allocates, on average, $\bar{n^*_j} = 15.5$ samples per test which strikes a balance between expenditure of $W_{\alpha}$ and $W_{\$}$.
}
\label{fig:gene-expression}
\end{figure*}

Figure~\ref{fig:gene-expression} shows the cost-aware and ERO decision rules on the prostate gene expression data set. 
The ERO method with $n=50$ selects many tests, but rapidly expends $W_{\$}$, as it does not optimize the sample size.
The ERO method with $n=3$ is able to test a much greater number of genes because it is limited in the amount of $W_{\$}$ expenditure per test.
While it may appear that the ERO method with $n=3$ is a much more favorable result, there are some critical concerns with this rejection sequence.
First, ERO ($n=3$) fails to reject nearly all of the tests in the first 50 genes. 
Among those genes are many that clearly have a strong differential gene expression signature, $\bar{x}_j$ shown in Figure~\ref{fig:gene-expression}.
Congruent with our observations of the effect of piggybacking in ERO investing (Figure~\ref{fig:piggy_backing}), had one of the early tests appeared later in the sequence, after ERO had accumulated a significant $\alpha$-wealth reserve, it would have been rejected.
Second, ERO ($n=3$) received a much noisier observation of the true differential gene expression signature compared to ERO ($n=50$).
As can be seen in Figure~\ref{fig:gene-expression}, ERO ($n=50$) does reject many early tests that do display a 2-fold increase in gene expression in the tumor compared to the normal cells when we observe all 50 samples.
ERO ($n=3$) likely fails to reject true alternative hypotheses, in part, due to the fact that it does not have access to enough samples to accurately assess the true differential expression level.
Cost-aware ERO allocates, on average $\bar{n^*_j} = 15.5$ samples per test.
This allocation strikes an balance between ERO with $n=3$ and $n=50$.
It is provides a more accurate measurement of the differential gene expression than ERO ($n=3$), it does not expend the sample collection resources as aggressively as ERO ($n=50$).
Finally, it is not as susceptible to the (random) ordering of the tests compared to ERO ($n=3$) which has the issue $\alpha$-piggybacking or ERO ($n=50$) which has the issue of \emph{betting the farm} in terms of sample resource wealth.

\subsection{LINCS L1000}
The Library of Integrated Network-Based Cellular Signatures (LINCS) NIH Common Fund program was established to provide publicly available data to study how cells respond to genetic stressors, such as perturbations by therapeutics \citep{geo}. The data considered is made up of L1000 assays of 1220 cell lines. The L1000 assay provides mRNA expression for 978 landmark genes. Differential gene expression is then calculated under a protocol known as level 5 preprocessing. \citet{cyclegan} infer the remaining genes using a CycleGAN and make the predictions available on their lab's webpage\footnote{\text{https://maayanlab.cloud/sigcom-lincs/\#/Download}}.

Data was prepared in a similar fashion to the prostate cancer data. Data was available for $1220$ samples which experienced a $10$ uM perturbation Vorinostat. 
Differential expression for $23,614$ genes against controls were processed as per the L1000 Level 5 protocol. Following this protocol, we divided the values by the standard deviation for each individual gene so that the data had unit variance. 
Our experimental protocol utilized 100 samples to estimate $q$. 
We set $q_j = 1 - \left( 1 + \exp(-\beta ([\bar{x_j}]_{1:100} - x_0) )\right)^{-1}$, where $x_0 = 2/\sigma$ and $\beta = 0.6$. 
The distribution of $q_j$ reflects our prior belief that most genes are likely to belong to the null hypothesis. Samples used for this estimation were shuffled between iterations. For both procedures $c_j=1, \forall j$, and $W_{\$}[0] = 100000$.
Our method was allowed $n \leq 1120$ samples while the ERO used $n = 1120$ samples. 
The order of genes was randomly shuffled and the procedures were repeated $1,000$ times to collect average statistics.
One sided Gaussian tests were performed where $\mu_0 = 0$ and $\mu_A = 0.5$. $\sigma = 1$ is assumed since data is already standardized.


Our method (CAERO) results in $797$ tests with $176.89$ rejections and an average sample size per test of $n=108$.
while ERO ($n=1120$) results in $90$ tests with $32.49$ rejections.  
The results on this data set are consistent with observations for the prostate cancer gene expression data set in that the ERO procedure expends its sample budget long before the $\alpha$-wealth has been exhausted.



\section{DISCUSSION}
\label{sec:discussion}
We have introduced an ERO generalized $\alpha$-investing procedure that has a self contained decision rule. 
This rule removes the need for a user-specified allocation scheme and optimally selects the sample size for each test. 
We have shown empirical results in support of the benefits of optimizing these testing parameters rather than being left to user choice. 

The cost-aware ERO methods does require the specification of the prior for the null, $q_j$.
We have shown that the number of tests and true rejections is not sensitive to variability in $q_j$, but is sensitive to bias in $q_j$ - for example, if the investigator is systematically optimistic.
For future work, it would be useful to investigate online learning methods for estimating $q_j$ and robust optimization formulations of the cost-aware ERO decision rule to reduce this sensitivity.

Cost-aware ERO does not, yet, have an explicit mechanism to hedge the risk of dollar wealth or $\alpha$-wealth loss. 
The current optimization problem assumes a risk-neutral player who wishes to not lose $\alpha$-wealth, on average, when conducting a test.
Since this desire is expressed in expectation, the variance of actual outcomes can be large, leading to $\alpha$-death without some constraint on the relative expenditure of $\alpha$-wealth.
For future work, it would be interesting to investigate a principled risk-hedging approach to conserve some wealth for future tests with the hope that a test with a more favorable reward structure is over the horizon.

The results from applying ERO and cost-aware ERO (Figure~\ref{fig:gene-expression}) show that there is a trade-off between expenditure of $W_{\alpha}$ and $W_{\$}$.
In our formulation of the problem, we have assumed that the initial $W_{\$}$ is fixed and can only decrease.
It would be interesting to apply a similar line of reasoning to $W_{\$}$ that was used to move from $\alpha$-spending to $\alpha$-investing.
Specifically, if a test is rejected, there may be some reward towards $W_{\$}$.
Then the decision rule may be modified to maximize $W_{\$}$ or to constrain it to be a martingale process as we have done here with $W_{\alpha}$.


\bibliography{biblio}

\clearpage

\appendix

\onecolumn

\section{THEORETICAL ANALYSIS OF LONG-TERM ALPHA-WEALTH AND COST-AWARE ERO SOLUTION}\label{sec:proof}

\subsection{Proof of Lemma~\ref{lemma:fs-alpha-bound}}
\begin{proof}[Proof of Lemma \ref{lemma:fs-alpha-bound}]
	The expected increment in $\alpha$-wealth is
	\begin{equation*}
		\mathbb{E}[W_{\alpha}(j) - W_{\alpha}(j-1)] = \mathbb{E}[R_j]\alpha - \mathbb{E}[1-R_j]\frac{\alpha_j}{1-\alpha_j}.
	\end{equation*}
	This equation requires the probability of rejection, which can be written in factorized form as
	\begin{equation*}
		\Pr(R_j=1) = \Pr(R_j=1 | \theta_j \in H_j) \Pr(\theta_j \in H_j) + \Pr(R_j=1 | \theta_j \not\in H_j) \Pr(\theta_j \not\in H_j).
	\end{equation*}
	Now, $\Pr(R_j=1 | \theta_j \in H_j) \leq \alpha_j$ by Assumption~\ref{eqn:assumption1} and $\Pr(R_j=1 | \theta_j \not\in H_j) \leq \rho_j$ by Assumption~\ref{eqn:assumption2}.
	Defining $\Pr(\theta_j \in H_j) = q_{j}$ gives the result.
\end{proof}

\subsection{Proof of Theorem~ \ref{theorem:long-term-alpha-wealth-scenario1}}

\begin{proof}
	
Since $\Theta_j=\{0,\bar{\theta}_j\}$, by lemma \ref{lemma:fs-alpha-bound} we have 
$$\mathbb{E} \left[W_{\alpha}(j)- W_{\alpha}(j-1) \mid W_{\alpha}(j-1)\right] = -\frac{\alpha_{j}}{1-\alpha_{j}}+\left[\rho_{j}-\left(\rho_{j}-\alpha_{j}\right) q_{j}\right]\left(\alpha+\frac{\alpha_{j}}{1-\alpha_{j}}\right)$$

We define $M_j := \rho_j - (\rho_j-\alpha_j) q_{j}$, then $\{W_\alpha(j) : j \in \mathbb{N}\}$ is submartingale, if and only if
\begin{equation*}
	M_j \geq \frac{\alpha_j/(1-\alpha_j)}{\alpha+\alpha_j/(1-\alpha_j)}.   
\end{equation*} 
Since $M_j = \rho_j - (\rho_j-\alpha_j) q_{j} >\rho_j(1-q_{j})$, thus $\{W_\alpha(j)\}$ is submartingale if
\begin{equation*}
\rho_j \geq \frac{\alpha_j/(1-\alpha_j)}{\alpha+\alpha_j/(1-\alpha_j)}\frac{1}{1-q_{j}}.
\end{equation*} 
\end{proof}

\subsection{Proof of Theorem~ \ref{theorem:long-term-alpha-wealth-scenario2}}
\begin{proof}
Let $M_j := \rho_j - (\rho_j-\alpha_j) q_{j}$, by Lemma~\ref{lemma:fs-alpha-bound}, $\{W_\alpha(j) : j \in \mathbb{N}\}$ is supermartingale, if
\begin{equation}\label{eq:nonincresing2}
	M_j \leq \frac{\alpha_j/(1-\alpha_j)}{\alpha+\alpha_j/(1-\alpha_j)}.  
\end{equation}

Next we define $s_j \in [0, (1-\alpha_j)/\alpha_j]$, a positive number to control how large the power is for the $j$th test, such that
$$\rho_j=s_j\alpha_j/(1-\alpha_j)$$
And we have
$$\rho_j-\alpha_j=[(s_j-1)\alpha_j+\alpha_j^2]/(1-\alpha_j)\geq (s_j-1)\alpha_j/(1-\alpha_j).$$
Thus,
$$M_j \leq \frac{s_j\alpha_j}{1-\alpha_j}-\frac{(s_j-1)\alpha_j}{1-\alpha_j} q_{j}.$$
The condition in \eqref{eq:nonincresing2} becomes
$$ \frac{1-\alpha_j}{\alpha_j} M_j \leq s_j-(s_j-1)q_j =s_j(1-q_j)+q_j \leq \frac{1}{\alpha+\alpha_j/(1-\alpha_j)}.$$
Thus, for a given $q_j$, the condition on $s_j$ for stochastically non-increasing wealth is
\begin{equation}
	s_j \leq  \left(\frac{1}{\alpha+\alpha_j/(1-\alpha_j)}-q_j\right)/(1-q_j).
	\label{eq:sj}
\end{equation}
The upper-bound in condition ~\eqref{eq:sj} is valid if it is positive. For $j$ large enough, if $\alpha_j/(1-\alpha_j)<\alpha$, then
$$\frac{1}{\alpha+\alpha_j/(1-\alpha_j)}-q_j>\frac{1}{2\alpha}-q_j.$$
If $\alpha_j<1/2$, this term is positive and the upper-bound  for $s_j$ is positive.
\end{proof}

\subsection{Proof of Theorem~ \ref{theorem:long-term-alpha-wealth-martingale}}
\begin{proof}
Following the notation of \citet{aharoni2014generalized}, the expected increment in $\alpha$-wealth is

\begin{equation}
    \mathbb{E}[W_{\alpha}(j) - W_{\alpha}(j-1) | W_\alpha(j-1)] = \mathbb{E}[R_j | \{R_1, \dots, R_{j-1} \} ](-\varphi + \psi_j) - \mathbb{E}[1-R_j| \{R_1, \dots, R_{j-1} \}](-\varphi).
    \label{eq:increment}
\end{equation}

By definition, $\alpha$-wealth is martingale when this increment is equal to $0$. Next, let $\mathbb{E}^{j-1}[X] = \mathbb{E}(X| \{X_1, \dots , X_{j-1}\}$. Then equation \ref{eq:increment} becomes:
\begin{equation}
    0 = \mathbb{E}^{j-1}[R_j]( -\varphi + \psi_j) + \mathbb{E}^{j-1}[1-R_j]( -\varphi )
    \label{eq:increment-2}
\end{equation}

Expanding $\mathbb{E}^{j-1}[R_j]$ in terms of variables in equation \ref{eqn:opt2} gives
\begin{equation*}
    \mathbb{E}^{j-1}[R_j] = \Pr(R_j=1) = \Pr(R_j=1 | \theta_j \in H_j) \Pr(\theta_j \in H_j) + \Pr(R_j=1 | \theta_j \not\in H_j) \Pr(\theta_j \not\in H_j).
\end{equation*}

By our previous assumptions, $\Pr(R_j=1 | \theta_j \in H_j) \leq \alpha_j$ and $\Pr(R_j=1 | \theta_j \not\in H_j) \leq \rho_j$. We define $\Pr(\theta_j \in H_j) = q_{j}$, and hence $\Pr(\theta_j \not\in H_j) = 1-q_{j}$.

Simplifying equation \ref{eq:increment-2} yields

\begin{equation}
    0 = (q_j \alpha_j + (1-q_j) \rho_j) (-\varphi_j + \psi_j) + (q_j(1-\alpha_j) + (1-q_j)(1 - \rho_j))(-\varphi_j)
    \label{eq:equalizing}
\end{equation} 

Solving equation \ref{eq:equalizing} for $\rho_j$
\begin{equation*}
    \rho_j = \left( \frac{1}{1-q_j} \right) \left( \frac{\varphi_j}{\psi_j} - q_j \alpha_j\right)
\end{equation*}

This implies that $\rho_j \propto \frac{\varphi_j}{\psi_j} - q_j \alpha_j$. This implies that the power of the test must balance the probability of rejection under the null and the ratio of the cost and reward of the test.  
\end{proof}

\subsection{Existence and Uniqueness of Solution}

Since the solution to the cost-aware ERO problem is infact an ERO solution, the existence of a solution is proven in Lemma 2 of \citet{aharoni2014generalized} given some assumptions which hold for a uniformly most powerful test with a continuous distribution function. Since these are the types of tests being considered in the current work, the necessary assumptions are met. 

\begin{theorem}
    In the cost-aware ERO solution with $\lambda = 0$, $\varphi$ is unique.
\end{theorem}
\begin{proof}
    Suppose $\exists$ a solution  $(\varphi_j^*, \psi_j^*, \alpha_j^*, \rho_j^*, n_j^*)$ such that
\begin{equation}
    \mathbb{E}(R_j) \psi_j = \mathbb{E}(R_j)^* \psi_j^*.
    \label{eq:expected-reward}
\end{equation}

Expanding the expectation of rejections in equation \ref{eq:expected-reward} yields
\begin{equation}
    (q_j \alpha_j + (1-q_j) \rho_j) \psi_j = (q_j \alpha_j^* + (1-q_j) \rho_j^*) \psi_j^*.
    \label{eq:expected-reward2}
\end{equation}
As per Lemma 2, the $\alpha$-wealth is martingale when using a solution to the cost-aware ERO optimization problem. Applying theorem 3 gives
\begin{equation}
    \rho_j = \left(\frac{1}{1-q_j}\right)\left(\frac{\varphi_j}{\psi_j} - q_j \alpha_j\right)
    \label{eq:rho1}
\end{equation}
\begin{equation}
    \rho_j^* = \left(\frac{1}{1-q_j}\right)\left(\frac{\varphi_j^*}{\psi_j^*} - q_j \alpha_j^*\right).
    \label{eq:rho2}
\end{equation}

Substituting equations \ref{eq:rho1} and \ref{eq:rho2} into equation \ref{eq:expected-reward2} gives

\begin{equation}
    q_j \alpha_j + (1-q_j) \left( \left(\frac{1}{1-q_j}\right) \left( \frac{\varphi_j}{\psi_j} - q_j \alpha_j \right) \right)\psi_j = q_j \alpha_j^* + (1-q_j) \left( \left(\frac{1}{1-q_j}\right) \left( \frac{\varphi_j^*}{\psi_j^*} - q_j \alpha_j^* \right) \right)\psi_j^*
\end{equation}
\begin{equation}
    \left( q_j \alpha_j + \frac{\varphi_j}{\psi_j} - q_j \alpha_j \right) \psi_j  =
     \left( q_j \alpha_j^* + \frac{\varphi_j^*}{\psi_j^*} - q_j \alpha_j^* \right) \psi_j^*
\end{equation}
\begin{equation}
    \left( \frac{\varphi_j}{\psi_j} \right)\psi_j=
    \left( \frac{\varphi_j^*}{\psi_j^*} \right)\psi_j^*
\end{equation}
\begin{equation}
    \varphi_j = \varphi_j^*
    \label{eq:unique-varphi}
\end{equation}
\end{proof}

Since the sample size, $n_j$ is now made a free parameter, a natural question is whether or not a unique $n_j$ can be selected. This is not necessarily the case. Consider the solution $(\varphi_j, \psi_j, \alpha_j, \rho_j, n_j)$ to the cost-aware ERO problem. Assume that a continuous distribution function is used. We now show that $(\psi_j, \alpha_j, \rho_j, n_j)$ are not necessarily unique.

 Suppose $\exists$ a solution  $(\varphi_j^*, \psi_j^*, \alpha_j^*, \rho_j^*, n_j^*)$ such that
\begin{equation}
    \mathbb{E}(R_j) \psi_j = \mathbb{E}(R_j)^* \psi_j^*.
\end{equation}
From equation \ref{eq:unique-varphi}, we know that $\varphi_j = \varphi_j^*$. From \citet{aharoni2014generalized}, any solution that is ERO must satisfy
\begin{equation}
    \frac{\varphi_j}{\rho_j} = \frac{\varphi_j}{\alpha_j} - 1.
    \label{eq:ero-ar1}
\end{equation}
Using equation \ref{eq:ero-ar1} it follows that
\begin{equation}
    \frac{\varphi_j^*}{\rho_j^*} = \frac{\varphi_j^*}{\alpha_j^*} - 1.
    \label{eq:ero-ar2}
\end{equation}
Solving equations \ref{eq:ero-ar1} and \ref{eq:ero-ar2} for $\varphi_j$ and $\varphi_j^*$ respectively give
\begin{equation}
    \varphi_j = \frac{1}{\frac{1}{\alpha_j} - \frac{1}{\rho_j}}
    \label{eq:ero-varphi1}
\end{equation}
\begin{equation}
    \varphi_j^* = \frac{1}{\frac{1}{\alpha_j^*} - \frac{1}{\rho_j^*}}
    \label{eq:ero-varphi2}
\end{equation}
Substituting equations \ref{eq:ero-varphi1} and \ref{eq:ero-varphi2} into equation \ref{eq:unique-varphi} gives
\begin{equation}
    \frac{1}{\frac{1}{\alpha_j} - \frac{1}{\rho_j}} = \frac{1}{\frac{1}{\alpha_j^*} - \frac{1}{\rho_j^*}}
\end{equation}
Simplifying gives
\begin{equation}
    \frac{1}{\alpha_j} - \frac{1}{\alpha_j^*} = \frac{1}{\rho_j} - \frac{1}{\rho_j^*}
    \label{eq:alpha-rho}
\end{equation}
Suppose $\alpha_j^* > \alpha_j$. It follows that $\rho_j^* > \rho_j$. Without loss of generality (with respect to the test statistic having a continuous distribution function), assume the test statistic is normally distributed. Writing out $\rho_j$ and $\rho_j^*$ explicitly then implies that
\begin{equation}
    1 - \Phi \left( z_{1 - \alpha_j^*} - \frac{\bar{\theta_j}}{\frac{\sigma_j}{\sqrt{n_j^*}}}\right) > 1 - \Phi \left( z_{1 - \alpha_j} - \frac{\bar{\theta_j}}{\frac{\sigma_j}{\sqrt{n_j}}}\right)
\end{equation}
\begin{equation}
    z_{1 - \alpha_j^*} - \frac{\bar{\theta_j}}{\frac{\sigma_j}{\sqrt{n_j^*}}} < z_{1 - \alpha_j} - \frac{\bar{\theta_j}}{\frac{\sigma_j}{\sqrt{n_j}}}
\end{equation}
\begin{equation}
    \frac{\bar{\theta_j}}{\frac{\sigma_j}{\sqrt{n_j}}} - \frac{\bar{\theta_j}}{\frac{\sigma_j}{\sqrt{n_j^*}}} < z_{1-\alpha_j} - z_{1-\alpha_j^*}
\end{equation}
\begin{equation}
    \sqrt{n_j} - \sqrt{n_j^*} < \frac{\sigma_j}{\bar{\theta_j}}\left( z_{1-\alpha_j} - z_{1-\alpha_j^*} \right)
    \label{eq:non-unique}
\end{equation}
Equation \ref{eq:non-unique} shows that a range on $n$ values can be used. In certain scenarios, this allows $(\psi_j, \alpha_j, \rho_j, n_j) \neq (\psi_j^*, \alpha_j^*, \rho_j^*, n_j^*)$. Considering the case when $\alpha_j^* < \alpha_j$ results in equation \ref{eq:non-unique} having the inequality reversed. Note that $n_j=1$ is not necessarily permitted by this range. Including $n_j$ in our problem is still useful, despite not being unique, since an a-priori specification may not yield the same maximal expected reward as leaving $n_j$ to be optimized.

\clearpage
\section{SIMULATION DETAILS}
\label{sec:simulation-setup}

In this section we describe simulations in greater detail so that our work can be fully reproduced. We briefly present the cost-aware ERO $\alpha$-investing method in algorithmic form. All baseline methods were based on initial values and code in \citet{2019onlineFDRpackage}.
\begin{algorithm}
\begin{algorithmic}
\State Input $\alpha$, $W_\alpha(0)$, $W_{\$}(0)$

\State $j \gets 0$
\While{$W_{\alpha}(j) > \epsilon$ and $W_{\$}(j) > \epsilon$}

\State Define $q_j$, $c_j$ for hypothesis $j$

\State Solve Problem 18 to obtain $\varphi_j$, $\alpha_j$, $\rho_j$, $\psi_j$, and $n_j$

\State Collect data $(x_{j1}, \ldots , x_{jn_j})$ and compute p-value $p_j$.

\If{$p_j \leq \alpha_j$}
\State $R_j \gets 1$
\Else 
\State $R_j \gets 0$
\EndIf

\State Update $W_{\$}(j+1) \gets W_{\$}(j) - c_j n_j$

\State Update $W_{\alpha}(j+1) \gets W_{\alpha}(j)  - \varphi_j + R_j \psi_j$
\State $j \gets j + 1$
\EndWhile
\end{algorithmic}
\caption{Cost-aware ERO Algorithm}
\end{algorithm}

We now provide a comparison of the usage of information of the hypothesis stream and prespecified parameters that each method considered in our simulation studies uses.

\defcitealias{javanmard2018online}{JM18}
\begin{table*}[h!]
    \centering
    \begin{tabular}{lrrrr}
\toprule
         &   Error Criterion & Params. Needed at Test & Prespecified Parameters & Incorporating Priors\\
  Method &          &               &             &    \\
\midrule
$\alpha$-investing &  mFDR & - & Prespecified spending scheme & Spending scheme\\
\midrule
ERO investing &  mFDR & Access to calculating $\rho$ & Prespecified spending scheme & Spending scheme\\
\midrule
LORD &  FDR (indep.), mFDR & - & $\gamma$, $w_0$, $b_0$ & Setting $\gamma$ \citepalias{javanmard2018online}\\
\midrule
SAFFRON &  FDR (indep.), mFDR & -& $\lambda$, $w_0$&  Adaptive\\
\midrule
\textbf{CAERO} & mFDR & $q_j$, Access to calculating $\rho$ & $a$, $\lambda$& At test\\

\bottomrule
\end{tabular}
    \caption{Comparison of information required for some online FDR-controlling methods.}
    \label{tbl:comparison}
\end{table*} 

\subsection{Experiment for Table \ref{tbl:sota}}

 For CAERO $(n^*)$, we set the lower bound on $\rho_j = 0.9$, $\lambda = 1e-3$, and $a=0.025$. For CAERO, $ n = 1$ we set the lower bound of $\rho_j = 0.01$. In our simulation we define $\alpha = 0.05$, $W_\alpha(0) = 0.0475$, $W_{\$}(0) = 1000$, $n_{iter} = 10000$ (number of iterations), $m = 1000$ (maximum number of tests per iteration), and $c = 1$ (cost per sample). $W_{\alpha}(0)$ for implementations of LORD and SAFFRON follow suggestions from \citet{javanmard2018online} and \citet{ramdas2018saffron}. An explicit algorithm is given in Algorithm \ref{alg:q9}. A similar experimental set up is used for Table \ref{tbl:sotaq1} and Table \ref{tbl:n10} where $q_j$ and $n_j$ are adjusted respectively. 

\begin{algorithm}
\caption{Simulation run in Table \ref{tbl:sota}}
\label{alg:q9}
\begin{algorithmic}
\State Input $\alpha$, $W_\alpha(0)$, $W_{\$}(0)$, $n_{iter}$, $m$, $c$
\For{$i=0$ to $i=n_{iter}$}
\State Set seed to $i$
\State $\boldsymbol{X} =$ [ ]
\For{$j = 0$ to $m$}
\State Sample $\theta_j$, with $Pr(\theta_j = 0) = q_j \sim Unif(.85,.95)$, and $Pr(\theta_j = 2) = 1 - q_j$.
\State $\boldsymbol{X}[j] = 1000$ realizations from $N(\theta_j, 1)$.
\EndFor
\For{each testing method (unique row in Table \ref{tbl:sota})}
\State $j \gets 0$
\While{ $W_{\alpha}(j) > \epsilon$ and $W_{\$}(j) > \epsilon$}
    \State $n_j = 1$ \Comment{Sample size to use if method not cost-aware.}
    \If{Spending and Investing rule separate}
        \State Obtain $\varphi_j$ from spending scheme.
        \State Obtain $\alpha_j$, $\psi_j$ from investing rule $\mathcal{I}$
    \Else{}
        \State Obtain $\varphi_j$, $\alpha_j$, $\psi_j$ from self-contained investing rule $\mathcal{I}$. (If using cost-aware, obtain and update $n_j$).
    \EndIf
    \State Perform 1-sided $Z$-test on $\boldsymbol{X}[j][0:n_j]$, and obtain $p$-value, $p_j$.
    \If{$p_j \leq \alpha_j$}
\State $R_j \gets 1$
\Else 
\State $R_j \gets 0$
\EndIf

\State Update $W_{\$}(j+1) \gets W_{\$}(j) - c n_j$

\State Update $W_{\alpha}(j+1) \gets W_{\alpha}(j)  - \varphi_j + R_j \psi_j$

\State $j \gets j + 1$
\EndWhile
\EndFor
\EndFor
\State Aggregate results
\end{algorithmic}
\end{algorithm}

\subsection{Experiment for Figure \ref{fig:simulation-results}}
We next discuss the experimental details for producing Figure \ref{fig:simulation-results}. For CAERO $(n^*)$, we set the lower bound on $\rho_j = 0.9$, $\lambda = 1e-3$, and $a=0.025$. For CAERO, $ n = 1$ we set the lower bound of $\rho_j = 0.01$. In our simulation we define $\alpha = 0.05$, $W_\alpha(0) = 0.0475$, $W_{\$}(0) = 1e8$, $n_{iter} = 2500$, $m = 1000$, and $c = 1$. An additional $q$, specifically $q_{1001}$ is drawn for solving the finite-horizon optimization problem when we reach the final test. We sample $q_j$ from a Beta$(a,100-a)$ distribution, and then sample whether $\theta_j$ is null or not based on the realization of $q_j$. This sampling scheme and relevant parameter values are given in Algorithm \ref{alg:fig1}.

\begin{algorithm}
\caption{Simulation run in Figure \ref{tbl:sota}}
\label{alg:fig1}
\begin{algorithmic}
\State Input $\alpha$, $W_\alpha(0)$, $W_{\$}(0)$, $n_{iter}$, $m$, $c$
\For{$i=0$ to $i=n_{iter}$}
\State Set seed to $i$
\For{$A\in \{10, 30, 50,  70, 90 \}$}
\State $\boldsymbol{X} =$ [ ]
\For{$j = 0$ to $m$}
\State Sample $q_j \sim Beta(A,100-A)$ 
\State Sample $\theta_j$, with $Pr(\theta_j = 0) = q_j$, and $Pr(\theta_j = 2) = 1 - q_j$.
\State $\boldsymbol{X}[j] = 1000$ realizations from $N(\theta_j, 1)$.
\EndFor
\For{each testing method (unique row in Figure \ref{fig:simulation-results})}
\While{ $W_{\alpha}(j) > \epsilon$ and $W_{\$}(j) > \epsilon$}
    \State Perform while loop in Algorithm \ref{alg:q9}.
\EndWhile
\EndFor
\EndFor
\EndFor
\State Aggregate results
\end{algorithmic}
\end{algorithm}

\subsection{Experiment for Figure \ref{fig:gene-expression}} The real data experiment shown in Figure \ref{fig:gene-expression} and detailed in Section \ref{sec:real-experiments} can be broken down into two steps: preprocessing and testing. 

In preprocessing, we load in two dataframes, one containing gene expression data for 50 normal (non-cancerous) samples ($6033 \times 50$), and a second containing similar data for 52 tumor samples ($6033 \times 52$). We take then mean across the normal samples to obtain a ($6033 \times 1$) vector containing the mean gene expression for normal patients. We calculate the standard deviation in a similar manner and use these vectors to standardize the $(6033 \times 52)$ dataframe containing tumor samples. Next, the first two columns of the tumor samples dataframe is separated from the remaining 50 columns to provide an informed estimate of $q_j$ for each test. It is important to note that we are allowing the potential for misspecification of $q_j$ by using an estimate of only two samples. Using these two samples:
$$
q_j = 1 - \left( 1 + \exp(-\beta ([\bar{x_j}]_{1:2} - x_0) )\right)^{-1},$$ where $x_0 = \log_{10}(4)/\hat{\sigma}$, $\beta = 2$, $[\bar{x_j}]_{1:2}$ denotes the sample mean of the two tumor samples separated from the remaining 50 tumor samples for the $j^{th}$ gene, and $\hat{\sigma}$ is the estimated standard deviation from the normal samples.

During the testing process, we perform a random shuffle of the genes and then run testing. This process is shown in Algorithm \ref{alg:fig2}. In this scenario, we set $\alpha = 0.05$, $W_\alpha(0) = 0.0475$, $W_{\$}(0) = 1000$, $n_{iter} = 1000$ (number of permutations), $m = 6033$, and $c = 1$. We set ERO investing to always use a sample size of $n_j = 50$. For cost-aware ERO, we set the lower bound of $\rho_j = 0.1$ and do not set a restriction on the upper value of $n_j$. However, if the optimized value is greater than 50, we choose to skip the test. Lastly, we set the constraint of $\varphi_j \leq 0.5*(1-q_j)W_{\alpha}(j)$ to avoid quick $\alpha$-death in some permutations. We note that this does not affect tests with large $q_j$ very much, as one might expect, since those tests require small bets in order to keep nature's strategy equalizing.

\begin{algorithm}
\caption{Simulation run in Figure \ref{fig:gene-expression}}
\label{alg:fig2}
\begin{algorithmic}
\State Input $\alpha$, $W_\alpha(0)$, $W_{\$}(0)$, $n_{iter}$, $m$, $c$
\For{$i=0$ to $i=n_{iter}$}
\State Set seed to $i$
\State Randomly shuffle data
\For{Method $\in \{$ ERO, cost-aware ERO$\}$}
    \While{ $W_{\alpha}(j) > \epsilon$ and $W_{\$}(j) > \epsilon$}
    \If{Method is cost-aware ERO}
         \State Set constraint $\varphi_j \leq (1-q_j)W_{\alpha}(j)$
         \State Solve Problem 18 to obtain $\varphi_j$, $\alpha_j$, $\psi_j$, and $n_j$
        \If{$0<n_j\leq50$}
            \State Perform test and update as per other simulations.
        \Else
            \State Skip test
        \EndIf
    \Else
        \State Perform test with sample size $n_j = 50$ and update as per other simulations.
    \EndIf
\EndWhile
\EndFor
\EndFor
\State Aggregate results.
\end{algorithmic}
\end{algorithm}

\clearpage
\section{EXTENSIONS OF COST-AWARE $\alpha$-INVESTING}
\label{sec:extensions}
In this section, we explore extensions of cost-aware ERO $\alpha$-investing.

\subsection{Cost tradeoffs.}
In Problem~\ref{eqn:opt2} the monetary cost does not factor in to the objective except through the constraints.
In many practical applications, it may be useful to simultaneously maximize the $\alpha$-reward and minimize the $\$$-cost.
In those applications, the objective function can be augmented to $\mathbb{E}(R_j) \psi_j - \gamma c_j n_j$, 
where $\gamma$ controls the trade-off between improving $\alpha$-wealth and minimizing $\$$-cost.

\subsection{Variable utility.}
Not all hypotheses may have equal value to the investigator and their value assessment may be independent of their assessment of the prior probability of the null hypothesis \citep{ramdas2017online}. For example, an investigator may be confident that a gene is differentially expressed in a particular tissue based on prior literature.
Then the prior probability that $\theta_j = 0$ is low, $p_j \approx 0$, and the utility of testing that hypothesis is also low.
There may be a different gene that has not been reported to be differentially expressed in the tissue, but if it is it would be a major scientific discovery.
Then, the investigator may assign a high prior probability to the null $\theta_j=0$, but also a high utility to the event that the null is rejected.
A generalized form of the cost-aware decision rule can be constructed to account for varying utility levels for each hypothesis in the batch by making the objective function $\sum_{j=1}^K \mathbb{E}_{\theta}(R_j)U(R_j)\psi_j$, where $U(R_j)$ is the utility of the rejection of the $j$-th null hypothesis.

\subsection{Batch testing.}
Many biological experiments are conducted in batches.
This scenario leads to a need for a decision rule that provides $(\alpha_j, \psi_j, n_j)_{j=1}^{K}$ for a batch of $K$ tests.
To address this need, the objective function in Problem~\ref{eqn:opt2} can be modified to $\sum_{j=1}^K \mathbb{E}_{\theta}(R_j)\psi_j$.
It seems reasonable to expend all of the $\alpha$-wealth for each batch and then collect the reward at the completion of the batch so that a next batch of hypotheses can be tested.
Therefore, we have constraints $\sum_{j=1}^K \varphi_j \leq W_{\alpha}(0)$ and $\sum_{j=1}^K c_j n_j \leq W_{\$}(0)$.
The other constraint remain and apply for each test in the batch.


\clearpage
\section{METHOD COMPARISON WITH $q=0.1$}
\label{sec:sota_q1}

In Table~\ref{tbl:sotaq1} we explore the comparison of cost-aware ERO investing with other methods for $q_j=0.1$. Naturally, when nulls occur infrequently, the issue of multiple testing is not as dire, and in some cases, FDR is controlled without using any correction \cite[Figure 6]{robertson2023online}. 
When true alternatives are abundant, cost-aware ERO requires a large ante ($\varphi_j$). In this simulation we set $a=1$ to highlight this effect. We also relax any lower bound on $\rho_j$.
This causes cost-aware ERO to rapidly deplete the $\alpha$-wealth.
In contrast, other methods do not increase the ante at all, or as severely, as cost-aware ERO.
However, it should be noted that the fraction of the tests that are true rejects among those that are performed is very high.
For example, in constant ERO investing the proportion of true rejects is 24\% and the proportion of true rejects for cost-aware ERO ($n_j \leq 10$) is 90\%.
This is a highly desirable result for the setting of biological experiments and other settings where sample cost is nontrivial.

\begin{table}[h!]
    \centering
    \begin{tabular}{llrrrr}
\toprule
          &         &    Tests &  True Rejects &  False Rejects &  mFDR \\
Scheme & Method &          &               &                &       \\
\midrule 	
constant & $\alpha$-spending &  10.0 &    2.49 &   0.00 &  0.001\\
          & $\alpha$-investing &   932.4 &  231.55 &   0.46 &  0.002\\
          & $\alpha$-rewards $k = 1$ &   925.0 &  230.13 &   0.46 &  0.002 \\
          & $\alpha$-rewards $k = 1.1$ &   926.5 &  221.54 &   0.42 &  0.002 \\
          & ERO investing &   934.3 &  230.87 &   0.45 &  0.002 \\
\midrule
relative & $\alpha$-spending &    66.0 &    4.95 &   0.00 &  0.001 \\
          & $\alpha$-investing &   994.0 &  661.55 &  10.19 &  0.015  \\
          & $\alpha$-rewards $k = 1$ &   989.2 &  416.37 &   2.00 &  0.005 \\
          & $\alpha$-rewards $k = 1.1$  &   991.4 &  626.93 &   7.47 &  0.012 \\
          & ERO investing &   994.8 &  820.57 &  34.14 &  0.040 \\
\midrule          
other & LORD++ &  1000.0 &  322.87 &   1.07 &  0.003 \\
          & LORD1 &  1000.0 &  93.81 &   0.07 &  0.001 \\
          & LORD2 &  1000.0 &  301.45 &   0.94 &  0.004\\
          & LORD3 &  1000.0 &  320.61 &   1.06 &  0.003\\
          & SAFFRON &  1000.0 &  779.92 &  23.92 &  0.030 \\
\midrule
  cost-aware & ERO $n_j = 1$ &     11.1 &    9.94 &   0.15 &  0.013 \\
  cost-aware & ERO $n_j \leq 10$ &    12.8 &   11.48 &   0.21 &  0.017\\
cost-aware & ERO $n_j^*$ &    10.6 &   9.48 &   0.13 &  0.012 \\
\bottomrule
\end{tabular}
    \caption{Comparison of cost-aware $\alpha$-investing with state-of-the-art sequential hypothesis testing methods with a prior probability of the null, $q = 0.1$ using 2,500 iterations. }
    \label{tbl:sotaq1}
\end{table}

\clearpage
\section{SENSITIVITY ANALYSIS WITH RESPECT TO $q_j$}
\label{sec:uncertainty-q}

Since the cost-aware ERO method makes use of the prior probability of the null hypothesis, $q_j$,  we investigate the sensitivity of the method to misspecification of that parameter.
Table~\ref{tbl:caero_abuncertain-q} shows the number of tests, mean true rejects, mean false rejects, and mFDR for simulation where the $q_j$ provided for optimization is misspecified. Specifically, we vary the specified $q_j$, when holding the true $q_j$ fixed at $0.9$. We performed $10,000$ iterations where cost-aware $\alpha$-investing is restricted to a single sample and $a=1$.

\begin{table}[h!]
    \centering
    \begin{tabular}{lllrrrr}
\toprule
          &    Tests &  True Rejects &  False Rejects &  mFDR \\
Specified q    & &  & & \\
\midrule
 0.50 & 2.7 & 0.13 & 0.06 & 0.049 \\
 0.70 & 10.6 & 0.33 & 0.06 & 0.047 \\	
 0.80 & 32.9 & 0.72 & 0.08 & 0.047 \\
 0.85 & 92.4 & 1.58 & 0.13 & 0.049 \\
 0.89 & 282.4 & 3.54 & 0.20 & 0.044 \\
 0.90 & 365.0 & 4.15 & 0.22 & 0.041 \\
\bottomrule
\end{tabular}
    \caption{Varying the magnitude of misspecification of $q_j$ shows that small deviations from the true value do not dramatically change performance, however, larger misspecifications result in fewer tests performed and fewer rejections. However, mFDR is still controlled.}
    \label{tbl:caero_abuncertain-q_2}
\end{table}

We now consider the effect on performance for the CAERO method presented in the main results. We draw $q_j \sim$ Unif$(0.65,0.95)$. We consider running the CAERO with the true $q_j$, $\hat{q}_j$ with $N(0,0.2)$ noise, and lastly negatively bias these $\hat{q}_j$ by $\{0.01,0.05,0.1,0.2,0.3,0.4\}$. For numerical stability we truncate all $\hat{q}_j \in [0.01,0.99]$.

\begin{table}[h!]
    \centering
    \begin{tabular}{lllrrrr}
\toprule
          &    Tests &  True Rejects &  False Rejects &  mFDR \\
Specified q    & &  & & \\
\midrule
 True & 230.5 & 39.13 & 0.43 & 0.011 \\
 \midrule
 Noisy $q$ & 230.3 & 39.09 & 0.21 & 0.005 \\	
 \midrule
 Noisy $q$, bias = -0.01 & 225.4 & 38.26 & 0.17 & 0.004 \\
 Noisy $q$, bias = -0.05 & 213.6 & 36.41 & 0.13 & 0.003 \\
 Noisy $q$, bias = -0.1 & 206.0 & 35.11 & 0.10 & 0.003 \\
 Noisy $q$, bias = -0.2 & 197.1 & 33.70 & 0.07 & 0.002 \\
 Noisy $q$, bias = -0.3 & 192.4 & 32.90 & 0.07 & 0.002 \\
 Noisy $q$, bias = -0.4 & 189.8 & 32.44 & 0.06 & 0.002 \\
\bottomrule
\end{tabular}
    \caption{When constraining $\varphi_j$ and allowing $n_j$ to be selected adaptively, the harmful effects of prior misspecification can be reduced. Bias in prior specification is more harmful than a noisy estimate.}
    \label{tbl:caero_abuncertain-q}
\end{table}


\clearpage
\section{COMPARISON WITH OTHER METHODS WITH $n_j=10$}
\label{sec:large_n}

The simulation study used in Table~\ref{tbl:sota} was repeated with setting $n=10$ for the existing methods and $n_j \leq 10$ for cost-aware ERO. 
The cost per sample was set to $c_j=1$, and the total budget was $W_{\$}[0] = 1000$. 
Hence, for methods that do not optimize sample size, the number of tests was limited to $100$. 

Cost-aware ERO performs more tests and rejects more true alternative hypotheses than existing methods. These results demostrate that even when state-of-the-art methods have access to larger sample sizes, the ability to optimize the sample size results in better performance.

\begin{table}[h!]
    \centering
    \begin{tabular}{llrrrr}
\toprule
          &         &    Tests &  True Rejects &  False Rejects &  mFDR \\
Scheme & Method &          &               &                &       \\
\midrule
constant  & spending        &   10.0 &   1.00 &  0.05 &  0.02 \\
          & investing       &   54.0 &   5.40 &  0.23 &  0.035 \\
          & rewards k = 1   &   47.8 &   4.79 &  0.20 &  0.034 \\
          & rewards k = 1.1 &   49.1 &   4.91 &  0.19 &  0.031 \\
          & ERO investing   &   54.0 &   5.40 &  0.23 &  0.035 \\
\midrule
relative  & spending        &   66.0 &   6.55 &  0.05 &  0.006 \\
          & investing       &   99.9 &  10.00 &  0.52 &  0.045 \\
          & rewards k = 1   &   99.9 &  10.00 &  0.45 &  0.039 \\
          & rewards k = 1.1 &   99.9 &  10.00 &  0.40 &  0.036 \\
          & ERO investing   &   99.9 &  10.00 &  0.52 &  0.045 \\
\midrule       
other     & LORD++          &   100.0 &  10.00 &  0.16 &  0.014 \\
          & LORD1           &  100.0 &  9.99 &  0.07 &  0.006 \\
          & LORD2           &  100.0 &  10.00 &  0.16 &  0.014 \\
          & LORD3           &  100.0 &  10.00 &  0.16 &  0.014 \\
          & SAFFRON         &  100.0 &  10.00 &  0.44 &  0.038 \\
\bottomrule
\end{tabular}
    \caption{Using a non-optimized pure strategy of setting n=10 (max n permissible in cost-aware simulations in Table~\ref{tbl:sota}). Although using a larger number of samples for each test gives more powerful tests, the number of samples used by cost-aware ERO investing is lower than all methods other than a constant spending scheme for $\alpha$-spending, which only gives a single true rejection.}
    \label{tbl:n10}
\end{table}

\newpage
\section{TWO-STEP COST AWARE ERO INVESTING}
\label{sec:two-step}
In order to set nature's strategy to equalizing in the two-step optimal procedure, the expected change in $\alpha$-wealth must be equal no matter what strategy the experimenter uses. It follows that the expected change in $\alpha$-wealth must be found for each strategy, and the system of equations solved.

In order for the martingale solution for the two-step game to hold, the following equations must hold.

\begin{multline}
    0 = P(TT_1)(-\varphi_1 - \varphi_2 + \alpha_1\psi_1 + \alpha_2\psi_2) + P(TT_2)(-\varphi_1 - \varphi_2 + \alpha_1\psi_1 + \rho_2 \psi_2) \\ +   P(TT_3)(-\varphi_1 - \varphi_2 + \rho_1\psi_1 + \alpha_2\psi_2) 
    + P(TT_4)(-\varphi_1 - \varphi_2 + \rho_1\psi_1 + \rho_2\psi_2)
\end{multline}
\begin{equation}
    0 = P(TS_1) (-\varphi_1 + \alpha_1\psi_1) + P(TS_2) (-\varphi_1 + \rho_1 \psi_1)
\end{equation}
\begin{equation}
    0 = P(ST_1)(-\varphi_2 + \alpha_2\psi_2) + P(ST_2)(-\varphi_2+ \rho_2\psi_2),
\end{equation}

where 

\begin{equation}
    P(TT_1) = (q_1)(q_2)
\end{equation}
\begin{equation}
    P(TT_2) = (q_1)(1-q_2)
\end{equation}
\begin{equation}
    P(TT_3) = (1-q_1)(q_2)
\end{equation}
\begin{equation}
    P(TT_4) = (1-q_1)(1-q_2)
\end{equation}
\begin{equation}
    P(TS_1) = q_1
\end{equation}
\begin{equation}
    P(TS_2) = 1-q_1
\end{equation}
\begin{equation}
    P(ST_1) = q_2
\end{equation}
\begin{equation}
    P(ST_2) = 1-q_2
\end{equation}

\end{document}